\documentclass[10pt, conference, letterpaper]{IEEEtran}
\IEEEoverridecommandlockouts

\usepackage{cite}
\usepackage{amsmath,amssymb,amsfonts}
\usepackage{graphicx}
\usepackage{textcomp}
\usepackage{xcolor}
\usepackage{array}
\usepackage{rotating}
\usepackage{multirow}
\usepackage{amsthm}
\usepackage{float} 
\usepackage{hyperref}

\usepackage[font={footnotesize}]{subcaption}
\usepackage{booktabs}
\usepackage{url}
\usepackage{verbatim}

\usepackage[shortlabels]{enumitem}

\def\circle#1{\raisebox{.5pt}{\textcircled{\raisebox{-.9pt} {#1}}}}

\usepackage[ruled, lined, noend, linesnumbered]{algorithm2e}
\SetKwBlock{RepeatNoEnd}{repeat}{}

\allowdisplaybreaks

\newtheorem{definition}{Definition}
\newtheorem{assumption}{Assumption}
\newtheorem{theorem}{Theorem}

\newtheorem{lemma}{Lemma}

\newtheorem{proposition}{Proposition}

\newcommand{\subparagraph}{}
 \usepackage{titlesec}
 \titlespacing\section{0pt}{6pt plus 2pt minus 1pt}{4pt plus 2pt minus 1.5pt}
 \titlespacing\subsection{0pt}{4pt plus 2pt minus 1pt}{2pt plus 2pt minus 1pt}
 \titlespacing\subsubsection{10pt}{2pt plus 0pt minus 2pt}{2pt plus 0pt minus 2pt}
\titleformat{\subsubsection}[runin]{\normalfont\normalsize\itshape}{\arabic{subsubsection})}{5pt}{}[:\,\,]

\usepackage{thmtools}

\declaretheoremstyle[%
 spaceabove=0pt,%
 spacebelow=4pt,%
 headfont=\normalfont\itshape,%
 postheadspace=1em,%
 qed=\qedsymbol%
]{proofstyle} 

\declaretheorem[name={Proof},style=proofstyle,unnumbered]{prf}

\def\BibTeX{{\rm B\kern-.05em{\sc i\kern-.025em b}\kern-.08em
    T\kern-.1667em\lower.7ex\hbox{E}\kern-.125emX}}
    
\newcommand{\Expect}{{\rm I\kern-.3em E}}
\newcommand{\Identity}{{\rm I\kern-.2em l}}

\begin{document}

\title{Adaptive Gradient Sparsification for Efficient Federated Learning: An Online Learning Approach \vspace{-0.1in}}

\author{
\IEEEauthorblockN{Pengchao Han\IEEEauthorrefmark{1}, Shiqiang Wang\IEEEauthorrefmark{2}, Kin K. Leung\IEEEauthorrefmark{1}}
\IEEEauthorblockA{\IEEEauthorrefmark{1}Department of Electrical and Electronic Engineering, Imperial College London, UK}
\IEEEauthorblockA{\IEEEauthorrefmark{2}IBM T. J. Watson Research Center, Yorktown Heights, NY, USA}
\IEEEauthorblockA{Email: hanpengchao199@gmail.com, wangshiq@us.ibm.com, kin.leung@imperial.ac.uk}
\thanks{This paper has been accepted at IEEE ICDCS 2020.

This research was sponsored in part by the U.S. Army Research Laboratory and the U.K. Ministry of Defence under Agreement Number W911NF-16-3-0001. The views and conclusions contained in this document are those of the authors and should not be interpreted as representing the official policies, either expressed or implied, of the U.S. Army Research Laboratory, the U.S. Government, the U.K. Ministry of Defence or the U.K. Government. The U.S. and U.K. Governments are authorized to reproduce and distribute reprints for Government purposes notwithstanding any copyright notation hereon.

P. Han was a visiting student at Imperial College London when contributing to this work.}
\vspace{-0.3in}
}

\maketitle

\begin{abstract}
Federated learning (FL) is an emerging technique for training machine learning models using geographically dispersed data collected by local entities. It includes local computation and synchronization steps. To reduce the communication overhead and improve the overall efficiency of FL, \emph{gradient sparsification} (GS) can be applied, where instead of the full gradient, only a small subset of important elements of the gradient is communicated. Existing work on GS uses a fixed degree of gradient sparsity for i.i.d.-distributed data within a datacenter. In this paper, we consider adaptive degree of sparsity and non-i.i.d. local datasets. We first present a fairness-aware GS method which ensures that different clients provide a similar amount of updates. Then, with the goal of minimizing the overall training time, we propose a novel online learning formulation and algorithm for automatically determining the near-optimal \emph{communication and computation trade-off} that is controlled by the degree of gradient sparsity. The online learning algorithm uses an estimated sign of the derivative of the objective function, which gives a regret bound that is asymptotically equal to the case where exact derivative is available. Experiments with real datasets confirm the benefits of our proposed approaches, showing up to $40\%$ improvement in model accuracy for a finite training time.
\end{abstract}

\begin{IEEEkeywords}
Distributed machine learning, edge computing, federated learning, gradient sparsification, online learning
\end{IEEEkeywords}

\section{Introduction} 
\label{sec:intro}

Modern consumer and enterprise users generate a large amount of data at the network edge, such as sensor measurements from Internet of Things (IoT) devices, images captured by cameras, transaction records of different branches of a company, etc. Such data may not be shareable with a central cloud, due to data privacy regulations and communication bandwidth limitation~\cite{kairouz2019advances}.
In these scenarios, \emph{federated learning}~(FL) is a useful approach for training machine learning models from local data~\cite{mcmahan2016communication,li2019federated,park2019wireless,yang2019federated,kairouz2019advances}. The basic process of FL includes local gradient computation at clients and model weight (parameter) aggregation through a server. Instead of sharing the raw data, only model weights or gradients need to be shared between the clients and the server in the FL process. Due to the exponential increase in the speed of graphic processing units (GPUs) including mobile GPUs~\cite{GPUGrowth,MobileGPU}, it is foreseeable that FL will be widely used in distributed artificial intelligence (AI) systems in the near future.

The clients in FL\footnote{Note that the original FL concept given in~\cite{mcmahan2016communication} only focuses on the consumer setting. We consider the extended FL definition in this paper that also includes enterprise and cross-organizational settings as described in~\cite{yang2019federated,kairouz2019advances}.} can range from mobile phones in the consumer setting~\cite{mcmahan2016communication} to edge servers and micro-datacenters in the enterprise or cross-organizational setting~\cite{yang2019federated,kairouz2019advances}. 
For different FL \emph{tasks} where each task trains a separate model, the involved types of clients and their network connection can be largely different. For example, one task can involve a number of mobile phones (consumer clients) within the same city, with fast networking but slow computation; another task can involve multiple micro-datacenters (enterprise clients) across the world, with slow networking but fast computation. The computation and networking overheads also vary with different types of models, learning algorithms, hyper-parameters (e.g., mini-batch size), etc., even for the same set of involved clients. Since both communication and computation consume a certain amount of time (and other types of resources such as energy), it is important to optimize the \emph{communication and computation trade-off} to minimize the model training time in FL.

In the original FL approach known as federated averaging (FedAvg), this trade-off is adjusted by the number of local update rounds between every two communication (weight aggregation) rounds~\cite{mcmahan2016communication}. After each step of local model update with gradient descent, FedAvg either sends all the model parameters or sends nothing. A more balanced approach that sends a sparse vector with a subset of important values from the full gradient, known as \emph{gradient sparsification} (GS), has recently gained attention in distributed learning systems~\cite{ALinear-Speedup-Analysis-of-Distributed-Deep-Learning}. Compared to the ``send-all-or-nothing'' approach in FedAvg, GS provides a higher degree of freedom for controlling the communication and computation trade-off.

Nevertheless, the degree of sparsity in existing GS approaches is fixed, which is not suitable for FL where the resource consumption can differ largely depending on the task, as explained above. Even for a single learning task, it is difficult to find the best degree of sparsity manually. The optimal sparsity depends on characteristics of the FL task, as well as the communication bandwidth and computational capability. 
In addition, existing GS algorithms mainly focus on cases where data is i.i.d.-distributed at clients (workers) within the same datacenter. Non-i.i.d. data distribution that frequently occurs in FL settings due to local data collection at clients has been rarely studied in the context of GS.

We therefore have the following open questions: 1) How to determine the optimal degree of sparsity for GS? 2) How to perform GS in FL with non-i.i.d. data and is it beneficial over the conventional send-all-or-nothing approach used in FedAvg? To answer these questions, we make the following main contributions in this paper.

\begin{enumerate}
    \item We present a fairness-aware bidirectional top-$k$ GS (FAB-top-$k$) approach for FL, where the sparse gradient vector includes $k$ elements derived from the original (full) gradients of all clients, \emph{both in the uplink} (client to server) \emph{and downlink} (server to client). The value of $k$ here can be regarded as a measure of the sparsity, where a smaller $k$ corresponds to a more sparse vector and requires less communication. This approach guarantees a minimum number of gradient elements used from each client. 
    \item We propose a novel formulation of the problem of adapting $k$ to minimize the overall training time (including computation and communication) in an \emph{online learning} setting where the training time is unknown beforehand. 
    \item A new online learning algorithm based on the estimated sign of the derivative of the cost function is proposed for solving the adaptive $k$ problem, and the regret bound of this algorithm is analyzed theoretically.
    \item The proposed approaches are evaluated using extensive experiments with real datasets, showing the benefits our approaches compared to other methods.
\end{enumerate}

Note that while we focus on training time minimization for ease of presentation, our proposed algorithm can be directly extended to the minimization of other types of additive resources, such as energy, monetary cost, or a sum of them. By controlling the sparsity degree $k$, we control the communication overhead and hence the communication and computation trade-off, because the overall training time (or other resource) is split between communication and computation.

\section{Related Work}\label{sec:related_work}

Since FL was first proposed in~\cite{mcmahan2016communication}, it has found various applications in the mobile and IoT domains~\cite{ICDCS19_DIOT,GuangxuZhu-Towards-an-Intelligent-Edge}. 
To reduce the communication bottleneck, methods have been proposed to find the appropriate times/sequences to communicate in~\cite{INFOCOM19-Round-Robin-Synchronization,MG-WFBP,Poseidon,ScalingDeepLearning}, which, however, do not reduce the overall amount of data to transmit. The communication and computation trade-off is adapted in~\cite{Gaia2017,WangJSAC2019,WangSysML2019,INFOCOM19-federated-wireless}, where after each local update step, either none or all the model weights are transmitted. An approach where only a subset of clients with relevant updates send their weights is considered in~\cite{CMFL}. These send-all-or-nothing approaches (from each client's perspective) may cause bursty communication traffic and do not consider the possibility of sending a sparse vector of model weights or gradient with low communication overhead.

GS is a way of compressing the gradient or model weight vector to improve communication efficiency.
In \emph{periodic averaging GS}~\cite{ALinear-Speedup-Analysis-of-Distributed-Deep-Learning,Gradient-Sparsification,Qsparse-local-SGD}, a random subset of gradient elements are transmitted and aggregated in each round, so that after a finite number of rounds, all elements of the full gradient vector are aggregated at least once.
In \emph{top-$k$ GS}, the $k$ gradient elements with the highest absolute values are transmitted and aggregated. For $N$ clients, the downlink transmission of the \emph{unidirectional} top-$k$ approach may include as many as $kN$ values~\cite{Deep-Gradient-Compression,Sparse-Communication-for-Distributed-Gradient-Descent,hardy2017distributed,The-Convergence-of-Sparsified-Gradient-Methods,Qsparse-local-SGD,chen2018adacomp,shi2019layer}, since different clients may select elements with different indices.
To avoid this issue, a global (\emph{bidirectional}) top-$k$ GS approach is proposed in~\cite{gtopk-Sparsification,gtopk-Convergence}, where the top-$k$ elements is selected among every pair of clients in an iterative way so that the downlink transmission includes at most $k$ values.
The above GS methods mainly focus on the datacenter setting with i.i.d. data distribution.

A few works apply GS and related techniques to FL with non-i.i.d. data. A random sparsification method similar to periodic averaging GS is proposed in~\cite{konevcny2016federated2}, which generally gives worse performance than top-$k$ GS (see Section~\ref{subsec:performance-fix-k}). A variant of bidirectional top-$k$ GS combined with quantization and encoding is recently developed in~\cite{Sattler2019}. It does not consider fairness among clients and could possibly exclude some clients' updates (also see Section~\ref{subsec:performance-fix-k}), which may cause the trained model to be biased towards certain clients. With the goal of reducing both communication and computation overheads, dropout and model pruning techniques have been applied~\cite{caldas2018expanding,jiang2019model,xu2019elfish}, which, however, may converge to a non-optimal model accuracy if an improper degree of sparsity is chosen. There exist other model compression techniques such as quantization~\cite{konevcny2016federated2}, which are orthogonal to GS and can be applied together with GS. We focus on GS in this paper.

In addition to the limitations mentioned above, most existing works on GS, model compression, and their variants (including those mentioned above) consider a fixed degree of sparsity. A few recent works consider thresholding-based adaptive methods in a heuristic manner without a mathematically defined optimization objective~\cite{chen2018adacomp,shi2019layer,xu2019elfish}. The focus of these works is to use different sparsity degrees in different neural network layers, which is orthogonal and complementary to our work in this paper. To the best of our knowledge, the automatic adaptation of sparsity (measured by $k$) \emph{with the objective of minimizing training time} has not been studied.

The optimal $k$ can depend on the communication bandwidth, computation power, model characteristics, and data distribution at clients. It is very difficult (if not impossible) to obtain a mathematical expression capturing the training convergence time with all these aspects, since even for simpler scenarios either not involving GS or not involving non-i.i.d. data, only upper bounds of the convergence have been obtained in the literature~\cite{WangJSAC2019,gtopk-Convergence,ALinear-Speedup-Analysis-of-Distributed-Deep-Learning,CooperativeSGD}.
In this paper, we use online learning to learn the near-optimal $k$ over time, which only requires mild assumptions (instead of an exact expression) of the convergence time. To our knowledge, we are the first to use online learning techniques to optimize the internal procedure of FL. Hence, our online learning formulation is new.

Furthermore, existing online learning algorithms either require the exact gradient/derivative of the cost function, which is difficult to obtain in practice~\cite{hazan2016introduction} or suffer from slow convergence if such information is not available (the bandit setting)~\cite{Online-convex-optimization-in-the-bandit,auer2002nonstochastic} (see further discussions in Section~\ref{sec:signEstimate}). It is challenging to develop an efficient online learning algorithm for determining $k$, which we address in this paper.

\textbf{Roadmap:}  Section~\ref{sec:fedLearn} describes FL using sparse gradients and our proposed FAB-top-$k$ GS approach. The online learning algorithm for finding the best $k$ and its theoretical analysis is presented in Section~\ref{sec:online-learning-to-determine-k}. The experimentation results are given in Section~\ref{sec:experimentation}. Section~\ref{sec:conclusion} draws conclusion.

\section{Federated Learning Using Sparse Gradients}
\label{sec:fedLearn}

\subsection{Preliminaries}

\label{subsec:preliminaries}

The goal of machine learning (model training) is to find the \emph{weights} (parameters) of the model that minimize a \emph{loss function}. Let $\mathbf{w}$ denote the vector of weights. The loss function $L(\mathbf{w}) := \frac{\sum_{h=1}^C f_h(\mathbf{w})}{C}$ captures how well the model with weights $\mathbf{w}$ fits the training data, where $f_h(\mathbf{w})$ is the loss for a data sample $h$, $L(\mathbf{w})$ is the overall loss, and $C$ is the number of data samples. The minimization of $L(\mathbf{w})$ is often achieved using stochastic gradient descent (SGD)~\cite{Goodfellow-et-al-2016}, where $\mathbf{w}$ is updated based on the \emph{estimated} gradient of $L(\mathbf{w})$ (denoted by $\nabla L(\mathbf{w})$) computed on a \emph{minibatch} of training data. 

In FL with $N$ different \emph{clients}, each client $i \in \{1,2,...,N\}$ has its own loss function $L(\mathbf{w},i):= \frac{\sum_{h=1}^{C_i} f_{i,h}(\mathbf{w})}{C_i}$, and the overall (global) loss function is $L(\mathbf{w}) := \frac{\sum_{i=1}^N C_i L(\mathbf{w},i)}{C}$, where $C_i$ denotes the amount of data samples available at client~$i$ and $C := \sum_{i=1}^N C_i$~\cite{WangJSAC2019}. The global loss function $L(\mathbf{w})$ is not directly observable by the system because the training data remains local at each client.

FL enables distributed model training without sharing the training data. The conventional FedAvg~\cite{mcmahan2016communication} approach includes performing a certain number of gradient descent steps at each client locally, followed by an aggregation of local model weights provided by all the clients through a central \emph{server}~\cite{mcmahan2016communication}. This procedure of multiple local update steps followed by global aggregation repeats until training convergence.

In this paper, we consider a slightly different procedure where instead of aggregating the model weights, we aggregate the sparsified gradients after every local update step. We will see in the experiments in Section~\ref{subsec:performance-fix-k} that with the same amount of communication overhead, our sparse gradient aggregation approach performs better than FedAvg.

Formally, in every \emph{training round} $m$, the model weight vector is updated according to
\begin{equation}\label{eq:gradientDescent}
    \mathbf{w}(m) = \mathbf{w}(m-1)-\eta \nabla_s L(\mathbf{w}(m-1))
\end{equation}
for $m=1,2,3,...$, where $\eta > 0$ is the SGD step size, $\mathbf{w}(m)$ is the weight vector obtained at the end of the current round $m$, $L(\mathbf{w}(m-1))$ is the loss obtained at the end of the previous round $m-1$ ($m=0$ corresponds to model initialization), and $\nabla_s L(\mathbf{w}(m-1)) \in \mathbb{R}^{D}$ is the sparse gradient of the global loss in round $m-1$ with $D$ defined as the dimension of the weight vector. For ease of presentation, we say that $\nabla_s L(\mathbf{w}(m-1))$ is computed in round $m$, and write $L(\mathbf{w}(m))$ as $L_m$.
Note that different from FedAvg, our $\mathbf{w}(m)$ at all clients are always synchronized, because all clients update their weights in (\ref{eq:gradientDescent}) using the same $\nabla_s L(\mathbf{w}(m-1))$. 
The computation of $\nabla_s L(\mathbf{w}(m-1))$ involves communication between clients and the server which is explained in Section~\ref{subsec:fairness-aware-top-k}.

\emph{Remark:} Note that both FedAvg~\cite{mcmahan2016communication} and our GS-based FL method (as described above) use synchronous SGD, which is beneficial over asynchronous SGD in FL settings with non-i.i.d. data distribution as discussed in~\cite{WangJSAC2019}.

\subsection{Fairness-Aware Bidirectional Top-$k$ GS} \label{subsec:fairness-aware-top-k}

The main goal of GS is to exchange only a small number of important elements in the gradient vector of each client, based on which the server computes a sparse global gradient that is sent to each client. 
In the following, we present a fairness-aware bidirectional top-$k$ GS (FAB-top-$k$) approach, where ``bidirectional top-$k$'' here indicates that \emph{both the uplink} (client to server) \emph{and downlink} (server to client) communications transmit only $k$ elements of the gradient vector. Compared to the unidirectional top-$k$ GS approach where the downlink may transmit as many as $kN$ (instead of $k$) elements~\cite{Deep-Gradient-Compression}, we save the downlink communication overhead by up to a factor of $N$, which is significant since $N$ can be large in FL. Compared to other approaches where the downlink transmits $k$ elements, such as~\cite{gtopk-Sparsification,Sattler2019}, our approach ensures fairness among clients in the sense that each client contributes at least $\lfloor k/N \rfloor$ elements to the sparse global gradient, which is useful for FL since the data at clients can be non-i.i.d. and biased. We use $\lfloor \cdot \rfloor$ and $\lceil \cdot \rceil$ denote the floor (rounding down to integer) and ceiling (rounding up to integer), respectively.

In FAB-top-$k$, similar to other GS approaches~\cite{Deep-Gradient-Compression,gtopk-Sparsification}, each client $i$ keeps an \emph{accumulated local gradient} denoted by $\mathbf{a}_i$. At initialization, each client $i$ sets $\mathbf{a}_i = \mathbf{0}$, where $\mathbf{0}$ is the zero vector. Then, for every round $m=1,2,3,...$, each client $i$ computes the full gradient $\nabla L(\mathbf{w}(m-1),i)$ locally and adds it to $\mathbf{a}_i$. Afterwards, it identifies the indices $\mathcal{J}_i$ of the top-$k$ absolute values of $\mathbf{a}_i$, and transmits these $k$ index-value pairs $\mathcal{A}_i := \{(j, a_{ij}): j\in \mathcal{J}_i\}$ to the server, where we use $a_{ij}$ to denote the $j$-th element of $\mathbf{a}_i$. After receiving $\mathcal{A}_i$ from every client $i$, the server identifies $k$ gradient elements that is aggregated and sent to the clients. 
The uniqueness of FAB-top-$k$ is the way the downlink $k$ elements are selected.

\subsubsection*{Fairness-Aware Gradient Element Selection}
Consider some $\kappa \leq k$, the server identifies the top-$\kappa$ elements from $\mathcal{A}_i$ received from client $i$, let $\mathcal{J}_i^\kappa$ denote the indices of these elements. Then, the server computes the union $\cup_i \mathcal{J}_i^\kappa$. 
Using a binary search procedure, we can find a value of $\kappa$ such that $\left| \cup_i \mathcal{J}_i^\kappa \right| \leq k$ and $\left| \cup_i \mathcal{J}_i^{\kappa+1} \right| > k$, where $|\cdot |$ here denotes the cardinality of the set. The indices in $\cup_i \mathcal{J}_i^\kappa$ are those gradient elements that will be aggregated and transmitted to clients in the downlink. If $\left| \cup_i \mathcal{J}_i^\kappa \right| < k$, we select $k - \left| \cup_i \mathcal{J}_i^\kappa \right|$ additional elements with the largest absolute values in $\left(\cup_i \mathcal{J}_i^{\kappa+1}\right) \setminus \left(\cup_i \mathcal{J}_i^\kappa\right) $
so that in total, $k$ elements are transmitted to clients. Let $\mathcal{J}$ denote the set of selected $k$ elements to be transmitted in the downlink. The server computes the aggregated gradient value $b_j := \frac{1}{C} \sum_i C_i a_{ij} \Identity[j\in\mathcal{J}_i]$ for each $j \in \mathcal{J}$, where $\Identity[\cdot]$ denotes the identity function that is equal to one if the condition is satisfied and zero otherwise. Then, the server sends $\mathcal{B} := \{(j,b_j):j\in\mathcal{J}\}$ to each client.

After the client receives $\mathcal{B}$ (and $\mathcal{J}$), each element indexed by $j$ in the sparse gradient is defined as $(\nabla_s L(\mathbf{w}(m-1)))_j := b_j \Identity[j\in\mathcal{J}]$, which is used to update the model weights using (\ref{eq:gradientDescent}). Then, each client $i$ resets $a_{ij}=0$ if $j\in\mathcal{J} \cap \mathcal{J}_i$.

It is easy to see that the above procedure provides a \emph{fairness guarantee} in the sense that each client contributes at least $\lfloor k/N \rfloor$ elements to the sparse gradient, because we always have $\left| \cup_i \mathcal{J}_i^\kappa \right| \leq k$ when $\kappa = \lfloor k/N \rfloor$.

The overall process is shown in Algorithm \ref{alg:fairnessTopK}. Note that Lines~\ref{algline:fairnessTopK:sparseGrad}--\ref{algline:fairnessTopK:gradientDescent} give the same results for all clients as they receive the same $\mathcal{B}$ from the server. Hence, $\mathbf{w}(m)$ remains synchronized among clients. We compute $\mathbf{w}(m)$ at clients instead of at the server so that we only need to exchange the sparse gradient. The sorting to obtain $\mathcal{J}_i$ at each client $i$ takes $O(D\log D)$ time. The computation of the set union $\cup_i \mathcal{J}_i^\kappa$ and the binary search of $\kappa$ to obtain $\mathcal{J}$ at the server takes $O(ND\log D)$ time, when sorted indices and values of $\mathcal{J}_i$ are computed once in every round $m$ and stored beforehand.

\emph{Remark:} Intuitively, FAB-top-$k$ converges due to the use of accumulated local gradient $\mathbf{a}_i$, which ensures that those gradient elements which are not included in the sparse gradient keep getting accumulated locally, so that they will be included in the sparse gradient if their accumulated values get large enough. Our experimentation results in Section~\ref{sec:experimentation} also confirm the convergence of FAB-top-$k$. A theoretical convergence analysis of FAB-top-$k$ is left for future work, while we anticipate that a similar analytical technique as in~\cite{gtopk-Convergence} can be used. 

We also note that the adaptive $k$ algorithm presented in the next section is \emph{not} limited to FAB-top-$k$ or the class of top-$k$ GS. It applies to any GS method with some sparsity degree. For simplicity, we refer to GS with $k$ elements in the sparse gradient vector as ``$k$-element GS'' in the following.

\begin{algorithm}[t]
\caption{FL with FAB-top-$k$} 
\label{alg:fairnessTopK} 

{\footnotesize

\KwIn{$k, \eta$}

\SetKwFor{EachClient}{each client $i = 1, ..., N$:}{}{}
\SetKwFor{TheServer}{the server:}{}{}

Initialize $\mathbf{w}(0)$ according to model specification and $\mathbf{a}_i \leftarrow \mathbf{0}$ ($\forall i$);

\For{$m = 1,... , M$}
{
    \EachClient{}{
        $\mathbf{a}_i \leftarrow \mathbf{a}_i + \nabla L(\mathbf{w}(m-1),i)$; \label{algline:fairnessTopK:gradientComputation}
        
        Compute $\mathcal{J}_i$;
        
        Send $\mathcal{A}_i := \{(j, a_{ij}): j\in \mathcal{J}_i\}$ to the server;  \label{algline:fairnessTopK:clientToServer}
    }
    
    \TheServer{}{
        Compute $\mathcal{J}$;
        
        \For{$j\in\mathcal{J}$}{
            $b_j \leftarrow \frac{1}{C} \sum_i C_i a_{ij}\Identity[j\in\mathcal{J}_i]$;
        }
    
        Send  $\mathcal{B} := \{(j,b_j):j\in\mathcal{J}\}$ to clients; \label{algline:fairnessTopK:serverToClient}
    }
    
    \EachClient{}{
        \For{$j=1,...,D$ \label{algline:fairnessTopK:sparseGrad}}{
            $\left(\nabla_s L(\mathbf{w}(m-1))\right)_j \leftarrow b_j \Identity[j\in\mathcal{J}]$;   
        }
        
        $\mathbf{w}(m) \leftarrow \mathbf{w}(m-1)-\eta \nabla_s L(\mathbf{w}(m-1))$;  \label{algline:fairnessTopK:gradientDescent}

        \For{$j\in\mathcal{J} \cap \mathcal{J}_i$}{
            $a_{ij}\leftarrow 0$;
        }
        
    }
}
}
\end{algorithm}

\section{Online Learning to Determine $k$}
\label{sec:online-learning-to-determine-k}

The choice of $k$ in $k$-element GS has a trade-off between communication-efficiency and learning-efficiency. A small $k$ requires a small amount of communication, but also causes the model to learn slowly because the direction of the sparse gradient can be very different from the direction of the full gradient in this case. Conversely, a large $k$ captures the gradient accurately, but also incurs a large communication overhead (time). It is therefore important to find the optimal $k$ to \emph{minimize the training convergence time} that is a sum of the computation and communication time in the FL process.

\subsection{Problem Formulation}

\subsubsection{Cost Definition}
We consider the training time (including computation and communication) of reaching a some desired value of the global loss function as the ``cost'' that we would like to minimize using an appropriately chosen value of $k$. The loss function value is related to the model accuracy, and we consider the loss instead of the accuracy because the loss is the direct objective used for model training, as explained in Section~\ref{subsec:preliminaries}. Our formulation and solution can be directly extended to other ``costs'' beyond the training time (such as energy consumption) as well, but we focus on the training time in this paper for simplicity to illustrate our ideas.

\begin{assumption}[Independent costs]
\label{assumption:independentCosts}
Consider any point during training where the model gives an arbitrary global loss of $L$. The progression of loss in subsequent training rounds (with $k$-element GS for some $k$) is independent of the value of $k'$ (for $k'$-element GS) used in the training rounds before reaching $L$. This also holds when multiple values of $k'$ and $k$ are used over time, before and after reaching $L$, respectively.
\end{assumption}

Assumption~\ref{assumption:independentCosts} says that the state of the model (captured by the weights) is reflected by the loss. We validate this with an experiment of FAB-top-$k$ with federated extended MNIST (\mbox{FEMNIST}) dataset~\cite{EMNIST} and $156$ clients (see Section~\ref{sec:experimentation} for further details). In Fig.~\ref{fig:AssumptionMotivationalExperiment}, we first perform FL with different values of $k$ before the global loss reaches a pre-defined target value $\psi$. Afterwards, we use $k=1000$. We see that regardless of the initial $k$, the losses after reaching $\psi$ (where we start to use $k=1000$ in all curves) remain almost the same, thus validating Assumption~\ref{assumption:independentCosts} empirically.
This assumption allows us to define the training time required to reach a loss $L$, when starting from loss $L'$, using $k$-element GS for some given $k$.

\begin{figure}[t]
\centering
\begin{subfigure}[b]{0.3\columnwidth}
 \centering
 \includegraphics[width=\textwidth]{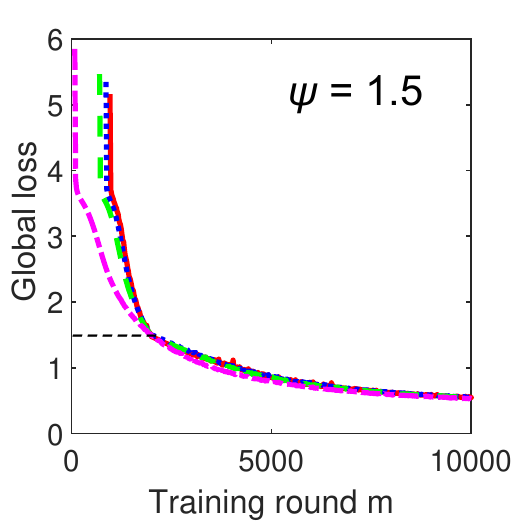}
\end{subfigure}
\begin{subfigure}[b]{0.3\columnwidth}
 \centering
 \includegraphics[width=\textwidth]{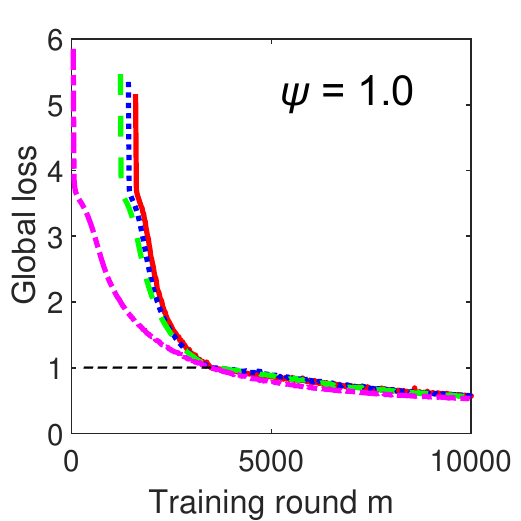}
\end{subfigure}
\begin{subfigure}[b]{0.33\columnwidth}
\centering
\includegraphics[width=1\textwidth]{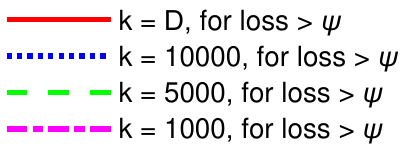}
\vspace{0.4in}
\end{subfigure}
\caption{Empirical validation of Assumption~\ref{assumption:independentCosts}. Recall that $D$ is the dimension of weight vector. For different $k$, training may start at different training rounds, so that all instances reach the target global loss $\psi$ at the same training round.}
\label{fig:AssumptionMotivationalExperiment}
\end{figure}

\begin{definition}[Training time for given loss interval]
\label{def:learningTime}
Define $\tilde{t}(k,l) \geq 0$ for any $k\in \{1,2,...,D\}$ and $l\in [L^*, L_0]$, such that for a training round with $k$-element GS starting with loss $L'$ and ending with loss $L < L'$, the total time (including computation and communication) of this round is equal to
\begin{equation}
\tilde{\tau}(L',L,k) := \int_{L}^{L'} \tilde{t}(k,l) dl
\label{eq:learningTimeDiscreteK}
\end{equation}
where $L^*$ denotes the optimal (minimum) global loss, and $L_0$ is the global loss at model initialization.
\end{definition}

Fig.~\ref{fig:costDef} gives an illustration of Definition~\ref{def:learningTime}. We express the training time in this integral form to facilitate the training time comparison later when using different values of $k$ over time. Note that different $k$ usually yields different sets of loss values obtained at the end of training rounds, but the total training time can be always expressed as the integral from the final loss to the initial loss according to Definition~\ref{def:learningTime}, by using the corresponding value of $k$ in $\tilde{t}(k,l)$ for each loss interval included in the integral. Next, we show that under mild assumptions, a definition of $\tilde{t}(k,l)$ always exists.

\begin{proposition}[Existence of $\tilde{t}(k,l)$]
\label{prop:tExistence}
For any $k\in\{1,2,...,D\}$, there always exists a function $\tilde{t}(k,l)$ that satisfies Definition~\ref{def:learningTime} when both of the following conditions hold:
\begin{enumerate}
    \item The sum of computation and communication time of one training round remains unchanged for any given $k$ (however, the time can be different for different $k$).
    \item When a training round (with some given $k$) starts at loss $L'$, the loss $L$ at the end of this round is a differentiable monotonically increasing function of $L'$ (i.e., function $L(L')$ decreases when $L'$ decreases) for any $k$.
\end{enumerate}
\end{proposition}

The proofs of the above proposition and subsequent theorems later in this paper are given in the appendix.

\begin{figure}
\centering
\includegraphics[width=0.9\columnwidth]{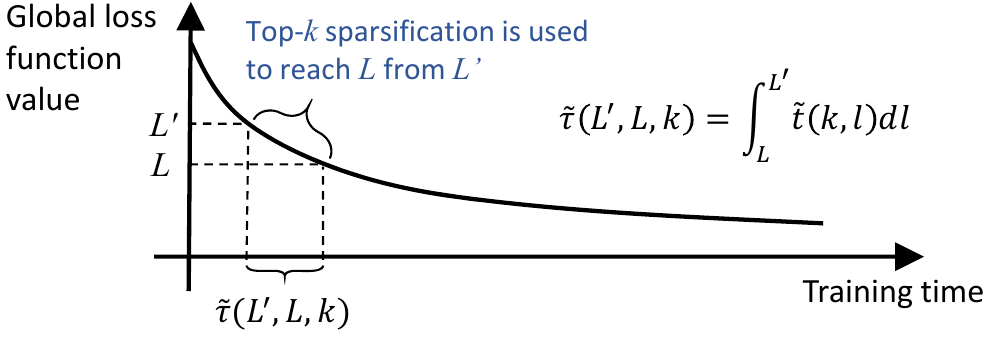}
\caption{Definition of training time for a loss interval.}
\label{fig:costDef}
\end{figure}

\subsubsection{Extension to Continuous $k$}
To facilitate the analysis and algorithm development later, we extend the definition of $\tilde{t}(k,l)$ to continuous $k$ as follows.

\begin{definition}[Randomized $k$-element GS] \label{def:randomizedTopK}
When $k$ is continuous in $[1,D]$, the system uses $\lfloor k \rfloor$-element GS with probability $\lceil k \rceil - k$, and $\lceil k \rceil$-element GS with probability $k - \lfloor k \rfloor$.
\end{definition}

This approach of rounding $k$ is known as stochastic rounding~\cite{gupta2015deep}. When $k$ is an integer, randomized $k$-element GS is equivalent to standard (non-randomized) $k$-element GS.
We focus on randomized $k$-element GS with continuous $k$ in the rest of this paper.

\begin{definition}[Expected training time for continuous $k$] \label{def:learningTimeContK}
Define
$t(k,l) := \left(\lceil k \rceil - k\right)\cdot \tilde{t}\left( \lfloor k \rfloor ,l \right)  + \left(k - \lfloor k \rfloor \right)\cdot \tilde{t}\left( \lceil k \rceil ,l \right)$
as the \emph{expected training time} for unit loss decrease. For a training round with \emph{randomized} $k$-element GS starting with loss $L'$ and ending with loss $L < L'$, the \emph{expected} total time (including computation and communication) of this round is
\begin{equation}
\tau(L',L,k):=\int_{L}^{L'} t(k,l) dl.
\label{eq:learningTimeContK}
\end{equation}
\end{definition}

\begin{assumption}[Properties of $t(k,l)$]
\label{assumption:convexity}
We assume that the following hold for $t(k,l)$:
\begin{enumerate}[a)]
\item (Convexity) The function $t(k,l)$ is convex in $k\in[1,D]$ for any given $l$. Consequently, $\tau(L',L,k)$ is also convex in $k$ for any given $L'$ and $L$ with $L' > L$.
\item (Bounded partial derivative) There exists some $g>0$, such that $\left| \frac{\partial t(k,l)}{\partial k}\right| \leq g$.
\item (Identical $k$ achieves minimum for all $l$) For any $l\neq l'$, we have $\arg\min_{k\in[1,D]} t(k,l) = \arg\min_{k\in[1,D]} t(k,l')$.
\end{enumerate}
\end{assumption}

Assumption~\ref{assumption:convexity} is only for the ease of presentation and regret analysis (see Definition~\ref{def:regret} below).
Although the value of $k$ yielding the minimum $t(k,l)$ ($\forall l$) is assumed to be the same in Item \emph{c)} of Assumption~\ref{assumption:convexity}, the value of $t(k,l)$ can be different for different $k$ and $l$. Since we do not make any statistical assumption on $t(k,l)$, our formulation belongs to the class of non-stochastic (adversarial) online learning~\cite{mannor2011bandits} with additional conditions given in Assumption~\ref{assumption:convexity}, which is more general (and usually more difficult) than stochastic online learning~\cite{Caron2012}.
From an empirical (practical) point of view, our algorithms presented later work even without Assumption~\ref{assumption:convexity}.

\subsubsection{Online Learning Formulation}
\label{subsec:onlineLearningFormulation}

Our goal is to find the optimal $k^*$ that minimizes the total training time of reaching some target loss value $L_M$, i.e., $k^* := \arg\min_k \int_{L_M}^{L_0} t(k,l) dl$.
However, the expression of $t(k,l)$ is unknown. The above definitions allow us to formulate the problem in the online learning setting where information related to $t(\cdot,l)$ gets revealed for different $l$ over time.
Sequential decisions of the choice of $k$ is made in every training round, and the effect of each choice is revealed \emph{after} the choice is made. 

Consider a sequence of choices $\{k_m:m=1,2,...,M\}$. In each round $m$, randomized $k_m$-element GS is used, where the training starts at loss $L'_m = L_{m-1}$, and at the end of the round, a new loss $L_m$ is obtained. The decision of $k_m$ is made based on the knowledge related to $t(k,l)$ for $l\in[L_{m-1}, L_0]$ and $k\in [1,D]$,  which has been revealed to the system before the beginning of the $m$-th round, while there is no knowledge about $t(k,l)$ for $l<L_{m-1}$. For simplicity, we denote $\tau_m(k) := \tau(L_{m-1},L_m,k)$ for short. 

It is important to note that in the definition of $\tau_m(k)$, the loss interval $[L_m,L_{m-1}]$ obtained in the $m$-th round when using $k_m$-element GS \emph{remains unchanged for a given $m$} regardless of the value of $k$ in $\tau_m(k)$. When $k\neq k_m$, $\tau_m(k)$ may correspond to the time that is not exactly one training round (and can possibly be a fractional number of training rounds), because we still focus on the same loss interval $[L_m,L_{m-1}]$ and the loss obtained exactly at the end of one training round if we had used $k$-element (instead of $k_m$-element) GS may be different from~$L_m$.

\begin{definition}[Regret]
\label{def:regret}
The regret~\cite{hazan2016introduction} of choosing $\{k_m\}$ compared to choosing the best $k^*$ in hindsight (i.e., assuming complete knowledge of $t(k,l)$ beforehand) is defined as
\begin{align*}
R(M) & \! := \!\! \sum_{m=1}^M \!\! \tau_{m}(k_m)\! - \!\!\! \int_{L_M}^{L_0} \!\!\!\! t(k^*,l) dl =\!\!  \sum_{m=1}^M \!\! \tau_{m}(k_m) \!-\!\!\! \sum_{m=1}^M\!\! \tau_{m}(k^*)
\end{align*}
where we note that $\int_{L_M}^{L_0} t(k,l) dl  = \sum_{m=1}^M \tau_{m}(k)$.
\end{definition}

The regret defined above is in fact an expected value due to the stochastic rounding of $k$, but we refer to this as the regret to distinguish from the expected regret where the expectation is over noisy estimations of the derivative sign that we will discuss later.
Our goal is to design an online learning algorithm for choosing $\{k_m\}$ such that the regret $R(M)$ grows sublinearly with $M$, so that the average regret over $M$ goes to zero as $M$ is large (i.e., $\lim_{M \rightarrow \infty} \frac{R(M)}{M} = 0$).

\emph{Remark:} The definition of training time in the form of an integral (Definitions~\ref{def:learningTime} and \ref{def:learningTimeContK}) and Assumptions~\ref{assumption:independentCosts} and \ref{assumption:convexity} are needed for a meaningful definition of the regret. The integral definition allows us to compare the training times although the sequence of losses obtained in training rounds and the total number of rounds for reaching the loss $L_M$ can be different when using $\{k_m\}$ and $k^*$. The comparison is possible because when using $k^*$-element GS, although $\tau_m(k^*)$ may correspond to a fractional number of rounds (the fraction can be either larger or smaller than one) for each $m$, $\sum_{m=1}^M\tau_m(k^*)$ is still the total time for reaching the final loss $L_M$. Assumption~\ref{assumption:independentCosts} and Item \emph{c)} in Assumption~\ref{assumption:convexity} ensure that the optimal solution is a static $k^*$ which does not change over time.

\subsection{Online Learning Based on the Sign of Derivative}

A standard approach of online learning in a continuous decision space is online gradient descent~\cite{hazan2016introduction}, which, however, is difficult to apply in our setting because it is hard to obtain an unbiased estimation of the gradient (equivalent to derivative in our case because our decision space for $k$ has a single dimension). We propose a novel online learning approach that only requires knowledge of the \emph{sign} of derivative instead of the actual derivative value.

\subsubsection{Online Learning Procedure for Determining $k_m$}

Define a continuous \emph{search interval} $\mathcal{K} := [k_{\textrm{min}}, k_{\textrm{max}}]$ to represent the possible interval for the optimal $k$ (i.e., $k^* \in \mathcal{K}$), where $k_{\textrm{min}}$ is usually a small integer larger than one to prevent ill-conditions in the gradient update when $k$ is too small, $k_{\textrm{max}}$ can be either the dimension of the weight vector (i.e., $D$) or a smaller quantity if we are certain that $k^*$ is within a smaller range (see Section~\ref{subsec:varyingSearchIntervals}). Let $B := k_{\textrm{max}} - k_{\textrm{min}}$.
Let $\mathcal{P}_\mathcal{K}(k)$ denote the projection of $k$ onto the interval $\mathcal{K}$, i.e., $\mathcal{P}_\mathcal{K}(k) := \arg\min_{k' \in \mathcal{K}} | k' - k |$.
We define the sign function as $\textrm{sign}(x) := \Identity[x>0] - \Identity[x<0]$. Note that with this definition, $\textrm{sign}(x) = 0$ if $x=0$.
Let $\tau'_m(k_m) := \left. \int_{L_m}^{L_{m-1}} \frac{\partial t(k,l)}{\partial k} dl \right|_{k=k_m}$ denote the derivative of $\tau_m(k)$ with respect to $k$ evaluated at $k=k_m$, and $s_m := \textrm{sign}( \tau'_m(k_m))$ denote the sign of the derivative. 
We also define $\delta_m := \frac{B}{\sqrt{2m}}$ as the step size for updating $k$ in round $m>0$ and define $\frac{1}{\delta_0} := 0$ for convenience. 

We propose an online learning procedure (given in Algorithm~\ref{alg:onlineLearning}) where the new value $k_{m+1}$ in the $(m+1)$-th step is determined from the derivative sign in the $m$-th step, by updating $k$ to the opposite direction of the derivative sign in Line~\ref{algline:onlineLearning:update} of Algorithm~\ref{alg:onlineLearning}.

\begin{algorithm}[t]
\caption{Online learning to determine $k$} 
\label{alg:onlineLearning} 

{\footnotesize

\KwIn{$k_\textrm{min}$, $k_\textrm{max}$, $B$, initial $k_1$}

\KwOut{$\{k_m\}$ in a sequential manner}

Set $\mathcal{K} \leftarrow [k_{\textrm{min}}, k_{\textrm{max}}]$;

\For{$m=1,2,...,M-1$}
{
Obtain $s_m$ (the sign of $\tau'_m(k_m)$) from the system;

Update $k_{m+1} \leftarrow \mathcal{P}_\mathcal{K}(k_m - \delta_m s_m)$, where $\delta_m := \frac{B}{\sqrt{2m}}$;  \label{algline:onlineLearning:update} 
}

}
\end{algorithm}

It is worth noting that in each step $m$, we only require that the sign of $\tau'_m(k_m)$ (i.e., $s_m$) is known to the system. The function $\tau_m(\cdot)$ itself or the loss values $L_{m-1}$ and $L_m$ are not known. We will show later in Section~\ref{sec:signEstimate} that only an estimated value of $s_m$ is necessary to obtain a similar regret bound (up to a constant factor). This makes it extremely easy to apply the algorithm in practice.

\subsubsection{Regret Analysis}
We first analyze the regret when the exact $s_m$ is obtained in each round $m$.
To facilitate the analysis, we assume that $t(k,l)$ for all $k\in [1, D]$ and $l \in [L^*, L_0]$ is given (but unknown) before the start of the system. This ensures that $\tau_m(k)$ does not change depending on the value of $k_m$ chosen in previous rounds. We also assume that the the difference between $L_{m-1}$ and $L_m$ is bounded by some finite value\footnote{Such a finite value always exists because the initial loss at model initialization $L_0$ is finite.} (although $\{L_m\}$ is not known to the system), hence according to Item \emph{b)} in Assumption~\ref{assumption:convexity}, we have \vspace{-2em}

{\small
\begin{align}
\left| \tau'_m(k) \right| & = \left| \int_{L_m}^{L_{m-1}}  \frac{\partial t(k,l)}{\partial k} dl \right| \leq \int_{L_m}^{L_{m-1}}\left|  \frac{\partial t(k,l)}{\partial k} \right| dl \nonumber \\ 
& \leq g (L_{m-1} - L_m) \leq G
\label{eq:GDef}
\end{align}
}%
where we define $G$ as the upper bound in the last inequality for any $m$, and the first equality is from the definition in (\ref{eq:learningTimeContK}).

\begin{theorem}
\label{theorem:RegretBound}
Algorithm~\ref{alg:onlineLearning} gives the following regret bound:
\begin{equation}
R(M) \leq G B \sqrt{2M}.
\end{equation}
\end{theorem}

\subsection{Using Estimated Derivative Sign}
\label{sec:signEstimate}

We now consider the case where the exact $s_m$ is not available and only an estimate is available. Let the random variable $\hat{s}_m \in \{-1,0,1\}$ denote the estimated sign of derivative in round $m$, which is used in Algorithm~\ref{alg:onlineLearning} in place of $s_m$. 
Since $\hat{s}_m$ ($\forall m$) is random, $k_m$ which depends on $\hat{s}_{m'}$ (for $m'<m$) is also random. Hence, $s_m$ is also a random variable that depends on $k_1,...,k_m$.
We assume that for any $m$, we have 
\begin{equation}
\mathrm{sign}\left(\Expect[\hat{s}_m | k_1,..,k_m]\right) = s_m,
\label{eq:equalExpectedSign}
\end{equation}
i.e., the sign of the expectation of $\hat{s}_m$ is equal to the derivative sign $s_m$, where $\Expect[\cdot]$ denotes the expectation. 
We also assume that there exists a constant $H_m \geq 1$ for each $m$, such that 
\begin{equation}
H_m \Expect[\hat{s}_m | k_1,..,k_m] = s_m
\label{eq:sExpect}
\end{equation}
and define $H$ such that $H_m \leq H$ for all $m$.

When $s_m \in \{ -1, 1 \}$ (i.e., the actual sign of derivative is not zero), condition (\ref{eq:equalExpectedSign}) holds if the probability of estimating the correct sign is higher than the probability of estimating a wrong sign (because $\hat{s}_m \in \{-1,0,1\}$), which is straightforward for any meaningful estimator. The difference between the probabilities of estimating the correct and wrong signs is captured by $H$ in (\ref{eq:sExpect}), where a larger $H$ corresponds to a smaller difference in the probabilities (i.e., a worse estimator), because if $\left|\Expect[\hat{s}_m | k_1,..,k_m]\right| = \left|\Pr\{\hat{s}_m =1| k_1,..,k_m\} - \Pr\{\hat{s}_m =-1| k_1,..,k_m\}\right|$ is small, a large $H$ is required since $s_m \in \{ -1, 1 \}$. When there is no estimation error, we have $H=1$. For $s_m = 0$, condition (\ref{eq:equalExpectedSign}) requires that the probabilities of (incorrectly) estimating as $\hat{s}_m = -1$ and $\hat{s}_m = 1$ are equal. Note that $s_m = 0$ almost never occurs in practice though.

\begin{theorem}
\label{theorem:ExpectedRegretBound}
When using the estimated derivative sign $\hat{s}_m$, Algorithm~\ref{alg:onlineLearning} gives the following expected regret bound:
\begin{equation}
\Expect[R(M)] \leq  G H B \sqrt{2M}.
\end{equation}
\end{theorem}

A specific way of estimating the sign of derivative in practice will be presented in Section~\ref{subsec:overallProcedure}.

\emph{Remark:} The regret bounds of using estimated and exact derivative signs only differ by a constant factor $H$. When considering $G$, $H$, and $B$ as constants, both approaches give a regret bound of $O(\sqrt{M})$, which is the same as the regret bound of online gradient descent with exact gradient~\cite{hazan2016introduction}. In addition, the time-averaged regret bound of our approach is $O\left( \frac{1}{\sqrt{M}} \right)$, which is the same as the convergence bound of gradient descent on an identical (unchanging) cost function~\cite{convex}. \emph{We can achieve the same asymptotic bound on changing cost functions using only the estimated sign of derivative}.

Compared to bandit settings that do not require any knowledge related to the gradient/derivative, our regret bound is asymptotically better than the continuous bandit case~\cite{Online-convex-optimization-in-the-bandit} and the same as the non-stochastic multi-armed bandit (MAB) case when restricting our decision space to integer values of $k$~\cite{auer2002nonstochastic}. However, the empirical performance of MAB algorithms applied to our problem is much worse than our proposed approach as we will see in Section~\ref{subsec:performance-online}, because MAB algorithms need to try each possible value of $k$ at least once to learn the effect of different $k$ that is used as a basis for selecting future $k$ values.

\subsection{Extension to Varying Search Intervals}
\label{subsec:varyingSearchIntervals}

The update step size $\delta_m$ and the regret bound $R(M)$ are proportional to the search range $B$. When the communication time is much larger than the computation time, a small value of $k$ is often beneficial. In this case, the update step $\delta_m$ in Algorithm~\ref{alg:onlineLearning} may be too large which causes high fluctuation of $k_m$, resulting in a large amount of time used for communication since $k_m$ can be large at times. To avoid this issue, we propose an extended online learning algorithm in Algorithm~\ref{alg:extendedOnlineLearning} where we reduce the search range (and hence the update step size) over time.

Algorithm~\ref{alg:extendedOnlineLearning} is equivalent to running multiple instances of Algorithm~\ref{alg:onlineLearning} with different search intervals $\mathcal{K}$ and corresponding $B$. When we are certain that the optimal $k$ is within a smaller interval, we may decide to use the smaller range (i.e., smaller $B$) and ``reset'' the counter $m$ in $\delta_m$ computation for evaluating subsequent values of $k_m$. To see why this can be beneficial, we consider two instances of Algorithm~\ref{alg:onlineLearning} with $B$ and $B'$ ($B' < B$), respectively. Assume both search intervals include $k^*$ but the smaller search interval is not known until running $M'$ rounds of the first instance. The total regret after $M'$ rounds of the first instance and $M''$ rounds of the second instance is upper bounded by $GH\sqrt{2}\left(B \sqrt{M'} + B' \sqrt{M''}\right)$, according to Theorem~\ref{theorem:ExpectedRegretBound}. Hence, after $M'$ rounds with $B$, if
\begin{equation}
B \sqrt{M'} + B' \sqrt{M''} < B \sqrt{M' + M''},
\label{eq:reduceB}
\end{equation}
then starting the second instance with $B'$ gives a lower overall regret bound.
By taking the square on both sides of (\ref{eq:reduceB}), cancelling $B^2 M'$, and dividing by $M''$, we can see that (\ref{eq:reduceB}) is equivalent to
$(B')^2 + 2 B B' \sqrt{\frac{M'}{M''}} < B^2$.
Hence, if (\ref{eq:reduceB}) holds for $M''=M'$, it also holds for any $M'' > M'$. For $M'' = M'$, (\ref{eq:reduceB}) becomes
$B' < B \left(\sqrt{2} - 1 \right)$.

In Algorithm~\ref{alg:extendedOnlineLearning}, we define an update window of $M_u$ rounds and consider the minimum/maximum values of $k_m$ obtained in this window divided/multiplied by a coefficient $\alpha$ to be the possible interval of $k^*$ (Lines~\ref{algline:extendedOnlineLearning:kMinMaxWindow}--\ref{algline:extendedOnlineLearning:kMinMaxAlpha}). After computing $B'$ for this new interval, Line~\ref{algline:extendedOnlineLearning:checkBUpdate} checks whether $B' < B \left(\sqrt{2} - 1 \right)$ is satisfied and whether the current instance has run for at least  the same number of rounds as the previous instance (i.e., $M'' \geq M'$). If both are true, it is beneficial to start a new instance according to the above discussion, and the algorithm starts a new instance by assigning the new interval in Line~\ref{algline:extendedOnlineLearning:assignNewInstance}. The variable $m_0$ in Algorithm~\ref{alg:extendedOnlineLearning} keeps track of when the new instance has started and acts equivalently to resetting the counter for $\delta_m$ computation in Line~\ref{algline:extendedOnlineLearning:update}.

\begin{algorithm}[t]
\caption{Extended online learning to determine $k$} 
\label{alg:extendedOnlineLearning} 

{\footnotesize

\KwIn{$k_{\textrm{min}}$, $k_{\textrm{max}}$, $B_0$, $\alpha \geq  1$, update window $M_u$, initial $k_1$}

\KwOut{$\{k_m\}$ in a sequential manner}

Initialize $m_0 \leftarrow 1$, $B \leftarrow B_0$, $n \leftarrow 0$, $M' \leftarrow 0$, $\mathcal{K} \leftarrow [k_{\textrm{min}}, k_{\textrm{max}}]$, $k'_\textrm{min} \leftarrow \infty$, and $k'_\textrm{max} \leftarrow 0$;

\For{$m=1,2,...,M-1$}
{
Obtain $\hat{s}_m$ (the estimated sign of $\tau'_m(k_m)$) from the system;

Update $k_{m+1} \leftarrow \mathcal{P}_\mathcal{K}(k_m - \delta_m \hat{s}_m)$, where $\delta_m := \frac{B}{\sqrt{2(m-m_0)}}$;  \label{algline:extendedOnlineLearning:update} 

$M'' \leftarrow m-m_0$; //Number of rounds running the current instance

$k'_\textrm{min} \leftarrow \min\left\{k'_\textrm{min}, k_{m+1}\right\}$, 
$k'_\textrm{max} \leftarrow \max\left\{k'_\textrm{max}, k_{m+1}\right\}$; \label{algline:extendedOnlineLearning:kMinMaxWindow} 

$n \leftarrow n+1$;  \label{algline:extendedOnlineLearning:addN} 

\If{$n \geq M_u$}{

$k'_\textrm{max} \leftarrow \min\left\{\alpha k'_\textrm{max}, k_{\textrm{max}}\right\}$, $k'_\textrm{min} \leftarrow \max\left\{ k'_\textrm{min} / \alpha, k_{\textrm{min}}\right\}$; \label{algline:extendedOnlineLearning:kMinMaxAlpha} 

$B' \leftarrow k'_\textrm{max} - k'_\textrm{min}$;

\If{$B' < \left(\sqrt{2} - 1\right)B$ {\normalfont\textbf{and}} $M'' \geq M'$\label{algline:extendedOnlineLearning:checkBUpdate} 
}{

$\mathcal{K} \leftarrow [k'_\textrm{min}, k'_\textrm{max}]$, $B \leftarrow B'$; \quad //Start new instance  \label{algline:extendedOnlineLearning:assignNewInstance} 

$M' \leftarrow M''$;  \quad //Current instance becomes previous

$m_0 \leftarrow m$;

}

$n \leftarrow 0$, $k'_\textrm{min} \leftarrow \infty$, $k'_\textrm{max} \leftarrow 0$;
}

}

}
\end{algorithm}

From the above discussion, we can see that if $M'' \geq M'$ at the last round $m=M-1$ in Algorithm~\ref{alg:extendedOnlineLearning}, the overall regret of Algorithm~\ref{alg:extendedOnlineLearning} for all $M$ rounds is upper bounded by the same bound given in Theorem~\ref{theorem:ExpectedRegretBound} (or Theorem~\ref{theorem:RegretBound} if exact derivative sign is used).
Depending on how the search interval shrinks over time, the actual regret of Algorithm~\ref{alg:extendedOnlineLearning} can be significantly better than that of Algorithm~\ref{alg:onlineLearning}.

\subsection{Implementation of Derivative Sign Estimation}
\label{subsec:overallProcedure}

To estimate the derivative sign, each client $i$ randomly selects one data sample $h$ from its minibatch in the current round $m$. The client computes three losses on this data sample: 1) the loss $f_{i,h}(\mathbf{w}(m-1))$ obtained at the end of the previous round $m-1$; 2) the loss $f_{i,h}(\mathbf{w}(m))$ obtained at the end of the current round $m$; 3) the loss $f_{i,h}(\mathbf{w}'(m))$, where $\mathbf{w}'(m)$ is the global weight vector obtained if instead of $k_m$-element GS, we use $k'_m$-element GS with $k'_m := k_m - \delta_m / 2$. We use the same data sample $h$ to compute these three losses so that they are comparable. The $k'_m$-element GS is used to evaluate whether it is beneficial to reduce the value of $k$. 

The losses $f_{i,h}(\mathbf{w}(m-1))$, $f_{i,h}(\mathbf{w}(m))$, and $f_{i,h}(\mathbf{w}'(m))$ are sent from each client to the server, and the server computes averages of the losses, denoted by $\tilde{L}(\mathbf{w}(m-1))$, $\tilde{L}(\mathbf{w}(m))$, and $\tilde{L}(\mathbf{w}'(m))$, respectively. 
Because the losses obtained using one round of $k_m$-element and $k'_m$-element GS are usually different (i.e., $\tilde{L}(\mathbf{w}(m)) \neq \tilde{L}(\mathbf{w}'(m))$), we need to map the time of one round when using $k'_m$-element GS to the time for reaching the loss $\tilde{L}(\mathbf{w}(m))$, as $\tau_m (k)$ is defined on the loss interval corresponding to $k_m$-element (instead of $k'_m$-element) GS (see Section~\ref{subsec:onlineLearningFormulation}). We estimate $\tau_m (k'_m)$ as \vspace{-0.11in}

{\small
\begin{equation}
\hat{\tau}_m\left(k'_m\right) := \theta_m\left(k'_m\right) \cdot \frac{\tilde{L}(\mathbf{w}(m-1)) - \tilde{L}(\mathbf{w}(m))}{\tilde{L}(\mathbf{w}(m-1)) - \tilde{L}(\mathbf{w}'(m))}
\label{eq:estTauExtension}
\end{equation}
}%
where $\theta_m\left(k'_m\right)$ is defined as the time of one training round when using $k'_m$-element GS. Note that $\tau_m (k'_m)$ (and $\hat{\tau}_m (k'_m)$) may correspond to the time for a fractional number of training rounds.
Then, the sign of derivative is estimated as \vspace{-0.11in}

{\small
\begin{equation}
\hat{s}_m = \textrm{sign}\left( \frac{\tau_m(k_m) - \hat{\tau}_m\left(k'_m\right)}{k_m - k'_m} \right)
\label{eq:signEstimation}
\end{equation}
}%
where the part inside $\textrm{sign}(\cdot )$ is the estimated derivative.

The above procedure is under the assumption that $\tilde{L}(\mathbf{w}(m-1)) > \tilde{L}(\mathbf{w}(m))$ and $\tilde{L}(\mathbf{w}(m-1)) > \tilde{L}(\mathbf{w}'(m))$, which holds for most of the time because a training iteration should decrease the loss. Occasionally, it may not hold due to randomness in minibatch sampling and choice of $h$ at each client. If it does not hold, (\ref{eq:estTauExtension}) has no physical meaning and we consider that $\hat{s}_m$ is unavailable and the value of $k_m$ remains unchanged in Algorithms~\ref{alg:onlineLearning} and \ref{alg:extendedOnlineLearning}. Lines~\ref{algline:extendedOnlineLearning:kMinMaxWindow} and \ref{algline:extendedOnlineLearning:addN} in Algorithm~\ref{alg:extendedOnlineLearning} are skipped when $k_m$ does not change in round $m$.

\begin{figure}
\centering
\includegraphics[width=1\columnwidth]{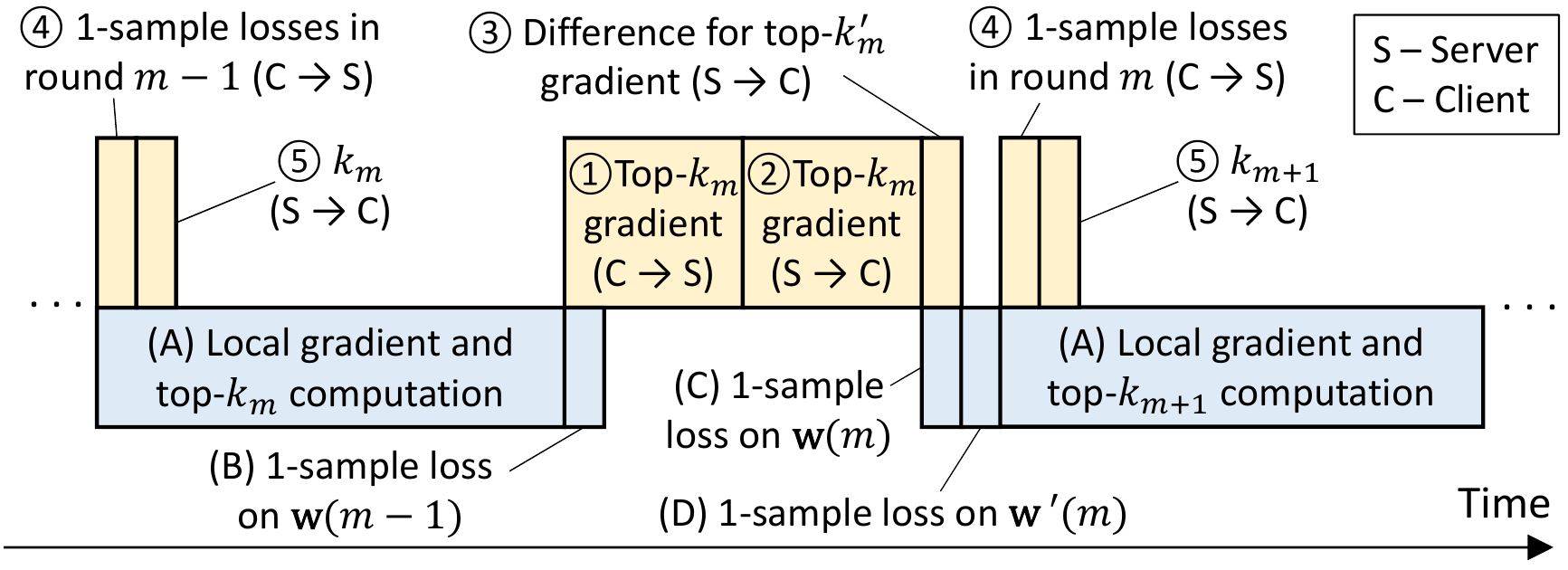}
\caption{Overall procedure, where  \circle{1}--\circle{5} show the communication between client and server, and (A)--(D) show the computation at each client. }
\label{fig:overallProcedure}
\end{figure}

The overall procedure is shown in Fig.~\ref{fig:overallProcedure}, where Step (A) corresponds to all the local computations at clients in Algorithm~\ref{alg:fairnessTopK}, Steps \circle{1} and \circle{2} correspond to Lines~\ref{algline:fairnessTopK:clientToServer} and \ref{algline:fairnessTopK:serverToClient} in Algorithm~\ref{alg:fairnessTopK}, respectively. Since the additional losses $\tilde{L}(\cdot)$ are computed only using one sample at each client, the additional computation time of each client (Steps (B), (C), (D)) is very small compared to the gradient computation on a minibatch in the training round (Step (A)).
Because $k'_m < k_m$, the $k'_m$-element GS result can be derived from $k_m$-element GS, hence only a small amount of information capturing the difference between $k_m$-element and $k'_m$-element GS results needs to be transmitted (Step \circle{3}) so that each client obtains $\mathbf{w}'(m)$ (in addition to $\mathbf{w}(m)$). The local losses $f_{i,h}(\cdot)$ obtained in round $m$ on the selected sample $h$ and the value of $k_{m+1}$ can be transmitted in parallel with the local gradient computation in the next round $m+1$ (Step \circle{4}), because the clients need to know the value of $k_{m+1}$ only after completing the local gradient computation (Line~\ref{algline:fairnessTopK:gradientComputation} in Algorithm~\ref{alg:fairnessTopK}) in round $m+1$. 
The server computes $k_{m+1}$ using $\hat{s}_m$ obtained from~(\ref{eq:signEstimation}) after receiving the losses from all clients in Step \circle{4}, and sends $k_{m+1}$ to clients in Step \circle{5}.
We ignore the server computation time in Fig.~\ref{fig:overallProcedure} because the server is usually much faster than clients and the time is negligible.

\section{Experimentation Results}
\label{sec:experimentation}

We evaluate our proposed methods with non-i.i.d. data distribution at clients using the \mbox{FEMNIST}~\cite{EMNIST} and CIFAR-10 datasets~\cite{CIFAR10}. FEMNIST includes $62$ classes of handwritten digits and letters. It is pre-partitioned according to the writer where each writer corresponds to a client in federated learning (hence non-i.i.d.). For FEMNIST, we consider $156$ clients with a total of $34,659$ training and $4,073$ test data samples. CIFAR-10 has $10$ classes of color images, with $50,000$ images for training and $10,000$ for test. For \mbox{CIFAR-10}, we consider a strong non-i.i.d. case with $100$ clients; each client only has one class of images that is randomly partitioned among all the clients with this image class. 
For both datasets, we train a convolutional neural network (CNN) that has the same architecture as the model in~\cite{WangJSAC2019} with over $400,000$ weights (i.e., $D > 400,000$). We fix the minibatch size to~$32$ and $\eta = 0.01$. The FL system is simulated, in which we define a \emph{normalized time} where the computation time in each round (for all clients in parallel) is fixed as $1$ and we vary the communication time of full gradient transmission\footnote{The communication time is defined as the time required for sending the entire $D$-dimensional gradient vector (both uplink and downlink) between all clients and the server. When sending less than $D$ elements of gradients, the communication time scales proportionally according to the actual number of elements sent, while assuming the uplink and downlink speeds are the same.}.  We mainly focus on FEMNIST except for the last experiment.

\subsection{Performance of FAB-top-$k$} \label{subsec:performance-fix-k}

We first evaluate our proposed FAB-top-$k$ approach with a fixed $k=1000$ and communication time of $10$. For comparison, we consider:
\begin{enumerate}
    \item \emph{Unidirectional top-$k$} GS where the downlink can include a maximum of $kN$ gradient elements~\cite{Deep-Gradient-Compression};
\item \emph{Fairness-unaware bidirectional top-$k$ (FUB-top-$k$)} GS that ignores the fairness aspect in FAB-top-$k$ and includes $k$ elements with largest absolute values in the downlink~\cite{gtopk-Sparsification}\footnote{Although this FUB-top-$k$ approach is similar with the global top-$k$ approach~\cite{gtopk-Sparsification}, note that we consider that all the gradients are transmitted to the server directly, because it is difficult to coordinate the direct exchange of gradients among pairs of clients in the FL setting due to firewall restrictions and possibly low bandwidth for peer-to-peer connection in WAN.},~\cite{Sattler2019};
    \item \emph{Periodic-$k$} GS that randomly selects $k$ elements~\cite{ALinear-Speedup-Analysis-of-Distributed-Deep-Learning,konevcny2016federated2};
    \item \emph{FedAvg} that sends the full gradient every $\left\lfloor D/(2k) \right\rfloor$ rounds\footnote{The division by $2$ is due to index transmission in GS.} which has the same average communication overhead as FAB-top-$k$ and FUB-top-$k$~\cite{mcmahan2016communication};
    \item \emph{Always-send-all} approach that always sends the full gradient in each training round $m$.
\end{enumerate}

The results in Fig.~\ref{fig:fix_k} show that FAB-top-$k$ performs better than all the other approaches, in terms of both the loss value and classification accuracy. In particular, the fact that we perform better than the send-all-or-nothing approach FedAvg~\cite{mcmahan2016communication} gives a positive answer to the second question in Section~\ref{sec:intro}. Compared to FUB-top-$k$ that gives a similar performance, our approach uses at least a certain number of gradient elements from each client and thus provides better fairness and avoids the possibility of some clients' data being completely ignored during the model training process (see Fig.~\ref{fig:fix_k} (right)).

\begin{figure}[t]
\centering

\begin{subfigure}[b]{0.9\columnwidth}
 \centering
 \includegraphics[width=\textwidth]{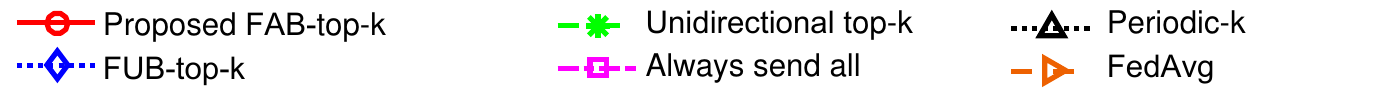}
\end{subfigure}
~
~

\begin{subfigure}[b]{0.3\columnwidth}
 \centering
 \includegraphics[width=\textwidth]{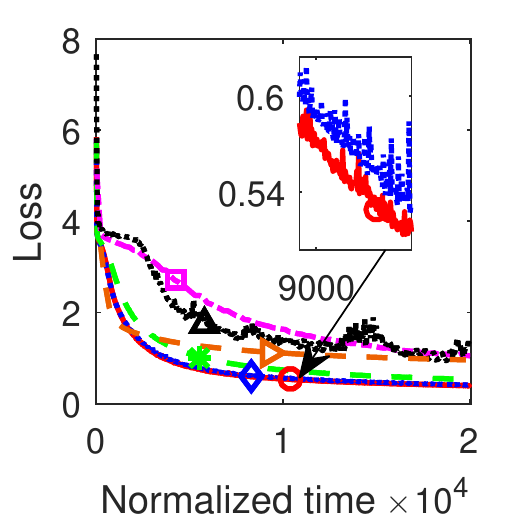}
\end{subfigure}
\begin{subfigure}[b]{0.3\columnwidth}
 \centering
 \includegraphics[width=\textwidth]{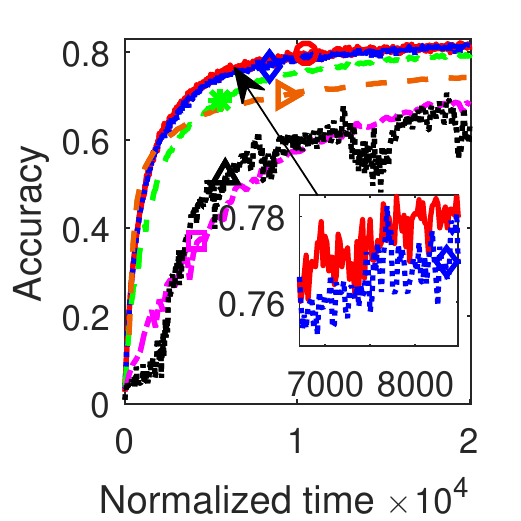}
\end{subfigure}
\begin{subfigure}[b]{0.33\columnwidth}
\centering
\includegraphics[width=1\textwidth]{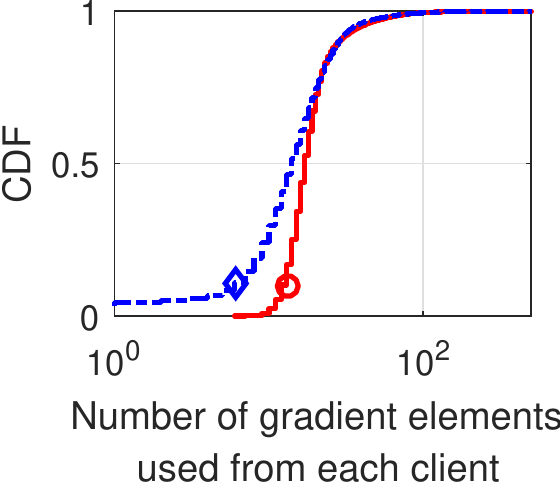}
\end{subfigure}
\caption{\label{fig:fix_k}  Performance of different GS methods with $k=1000$, communication time of $10$ on FEMNIST dataset. The markers on each curve are only used to map the curves to their legends, and the location of the marker on the curve is arbitrary and does not carry any specific meaning.}
\end{figure}

\begin{figure}[t]
\centering
\begin{subfigure}[b]{0.205\columnwidth}
 \centering
 \includegraphics[width=\textwidth]{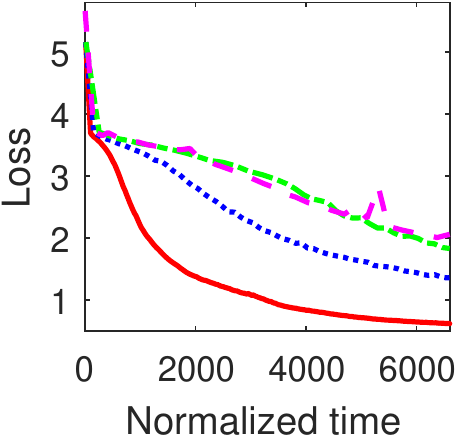}
\end{subfigure}
\begin{subfigure}[b]{0.22\columnwidth}
 \centering
 \includegraphics[width=\textwidth]{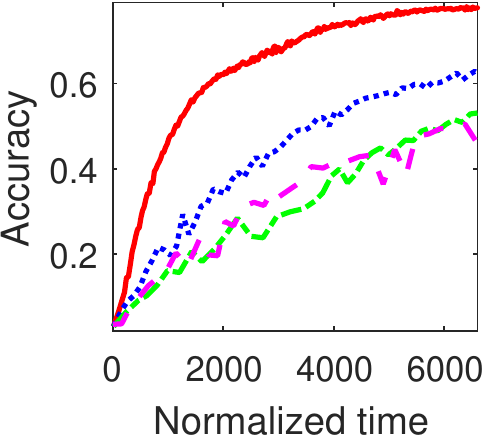}
\end{subfigure}
\begin{subfigure}[b]{0.37\columnwidth}
 \centering
 \includegraphics[width=\textwidth]{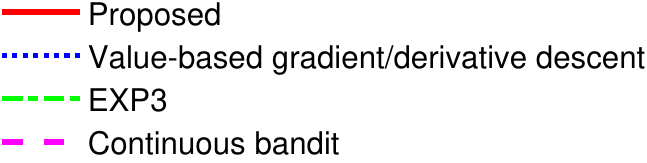}
 \vspace{0.06in}
\end{subfigure}

\begin{subfigure}[b]{0.24\columnwidth}
 \centering
 \includegraphics[width=\textwidth]{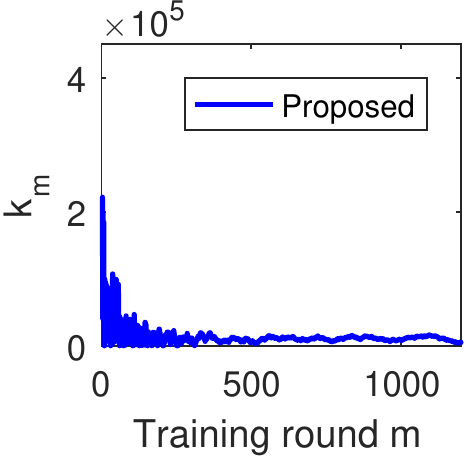}
\end{subfigure}
\begin{subfigure}[b]{0.24\columnwidth}
 \centering
 \includegraphics[width=\textwidth]{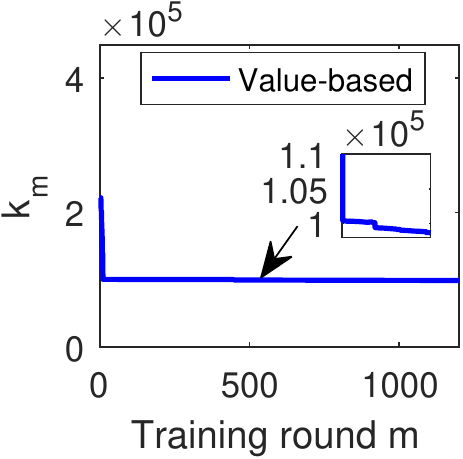}
\end{subfigure}
\centering
\begin{subfigure}[b]{0.24\columnwidth}
 \centering
 \includegraphics[width=\textwidth]{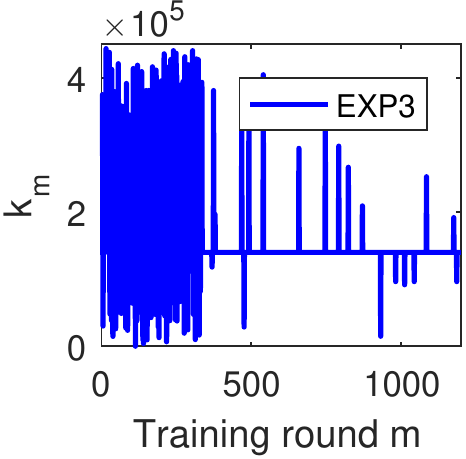}
\end{subfigure}
\begin{subfigure}[b]{0.24\columnwidth}
 \centering
 \includegraphics[width=\textwidth]{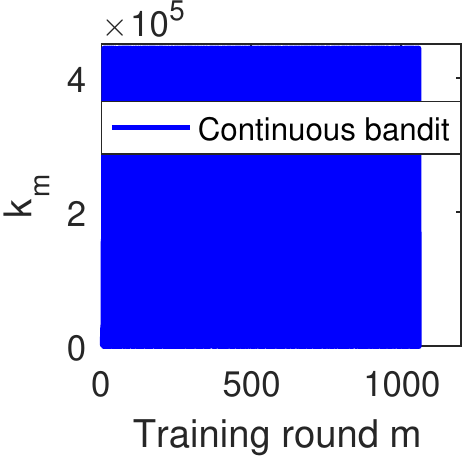}
\end{subfigure}
\caption{\label{fig:diffMethods} Performance of adaptive $k$ with different online learning methods (communication time: 10, dataset: FEMNIST).}
\end{figure}

\begin{figure}[t]
\centering

\begin{minipage}[b]{1\columnwidth}
\centering
\centering
\begin{subfigure}[b]{0.24\columnwidth}
 \centering
 \includegraphics[width=\textwidth]{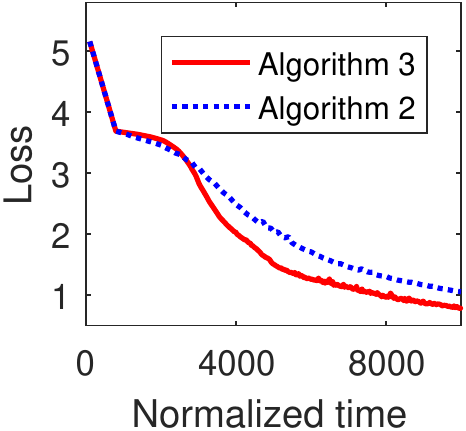}
\end{subfigure}
\begin{subfigure}[b]{0.24\columnwidth}
 \centering
 \includegraphics[width=\textwidth]{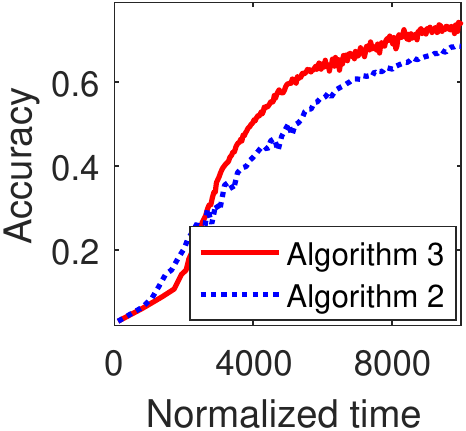}
\end{subfigure}
\centering
\begin{subfigure}[b]{0.24\columnwidth}
 \centering
 \includegraphics[width=\textwidth]{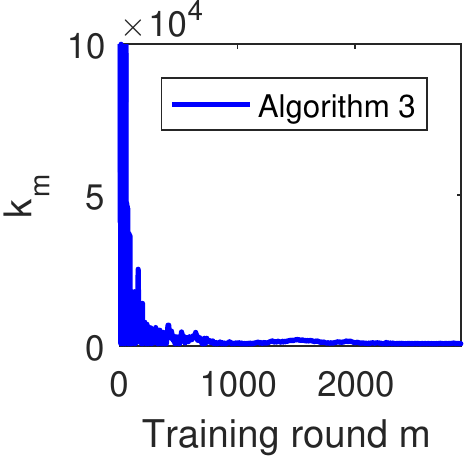}
\end{subfigure}
\begin{subfigure}[b]{0.24\columnwidth}
 \centering
 \includegraphics[width=\textwidth]{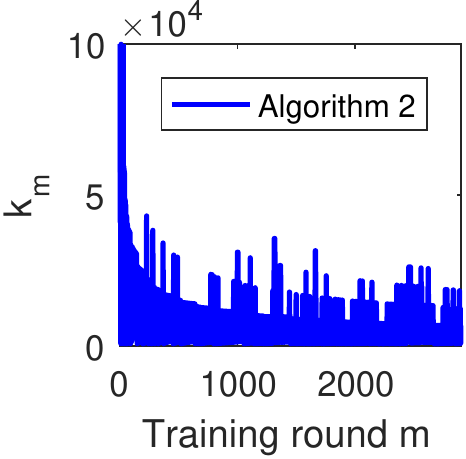}
\end{subfigure}
\caption{\label{fig:compSignOrig} Comparison between Algorithms~\ref{alg:onlineLearning} and \ref{alg:extendedOnlineLearning} (communication time: 100, dataset: FEMNIST).}
\end{minipage}
\vspace{-0.05in}

\begin{minipage}[b]{1\columnwidth}
\centering

\begin{subfigure}[b]{0.24\columnwidth}
 \centering
 \includegraphics[width=\textwidth]{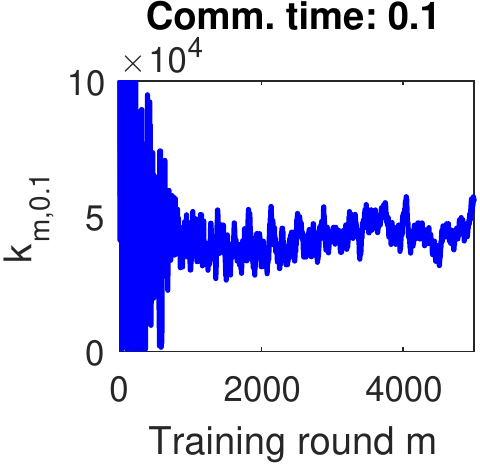}
\end{subfigure}
\begin{subfigure}[b]{0.24\columnwidth}
 \centering
 \includegraphics[width=\textwidth]{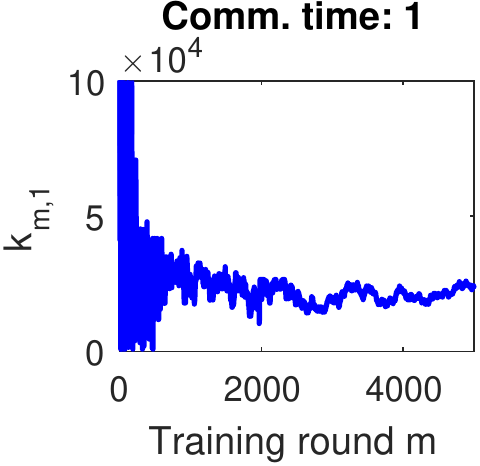}
\end{subfigure}
\begin{subfigure}[b]{0.24\columnwidth}
 \centering
 \includegraphics[width=\textwidth]{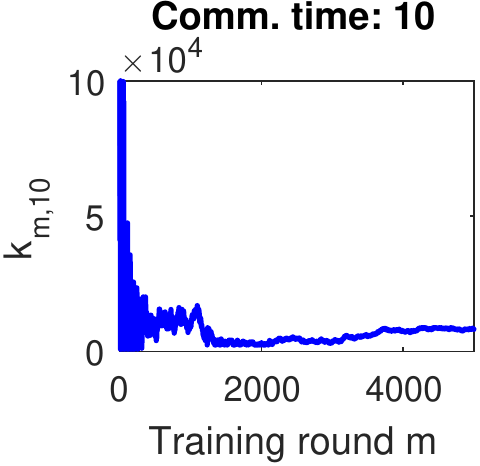}
\end{subfigure}
\begin{subfigure}[b]{0.24\columnwidth}
 \centering
 \includegraphics[width=\textwidth]{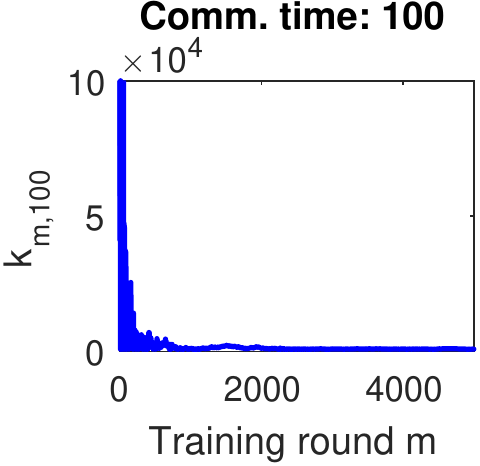}
\end{subfigure}

\begin{subfigure}[b]{0.5\columnwidth}
 \centering
 \includegraphics[width=\textwidth]{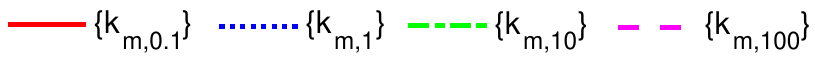}
\end{subfigure}
~
~

\begin{subfigure}[b]{0.255\columnwidth}
 \centering
 \includegraphics[width=\textwidth]{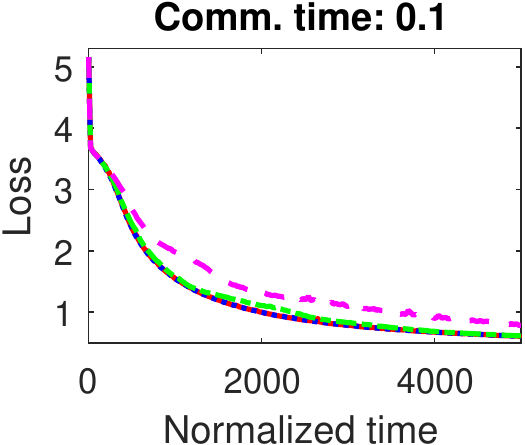}
\end{subfigure}
\begin{subfigure}[b]{0.22\columnwidth}
 \centering
 \includegraphics[width=\textwidth]{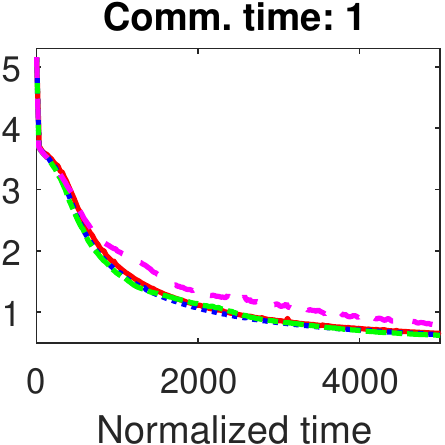}
\end{subfigure}
\begin{subfigure}[b]{0.22\columnwidth}
 \centering
 \includegraphics[width=\textwidth]{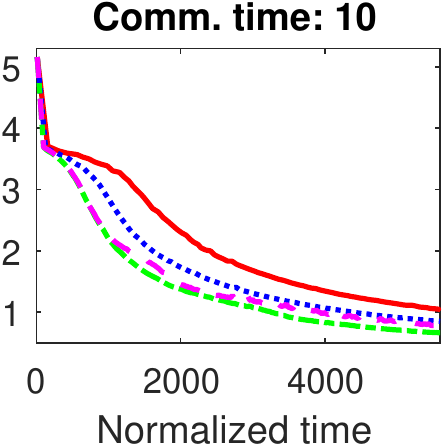}
\end{subfigure}
\begin{subfigure}[b]{0.22\columnwidth}
 \centering
 \includegraphics[width=\textwidth]{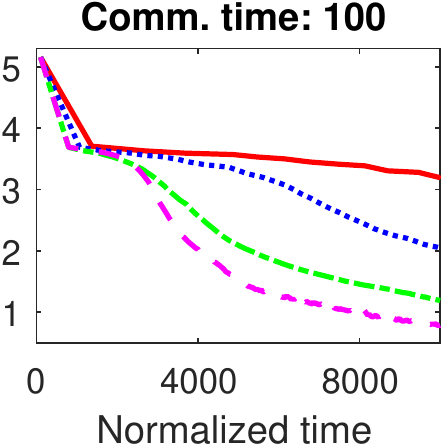}
\end{subfigure}
\vspace{-0.06in}

\begin{subfigure}[b]{0.242\columnwidth}
 \centering
 \includegraphics[width=\textwidth]{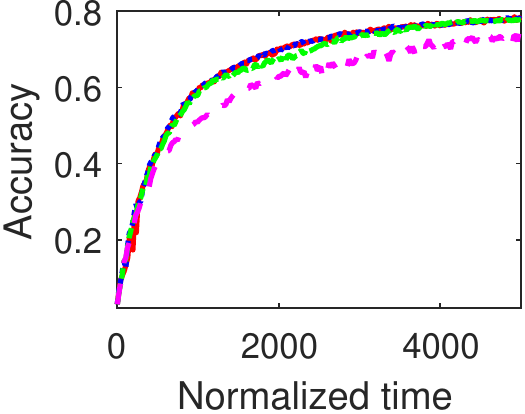}
\end{subfigure}
\begin{subfigure}[b]{0.22\columnwidth}
 \centering
 \includegraphics[width=\textwidth]{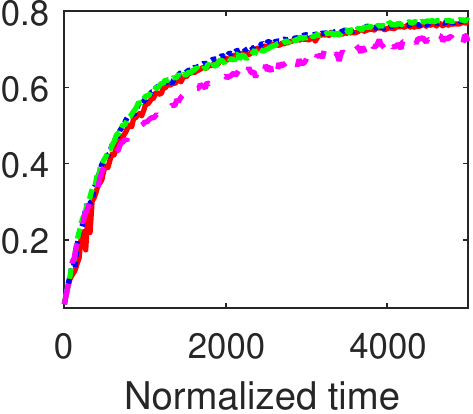}
\end{subfigure}
\begin{subfigure}[b]{0.22\columnwidth}
 \centering
 \includegraphics[width=\textwidth]{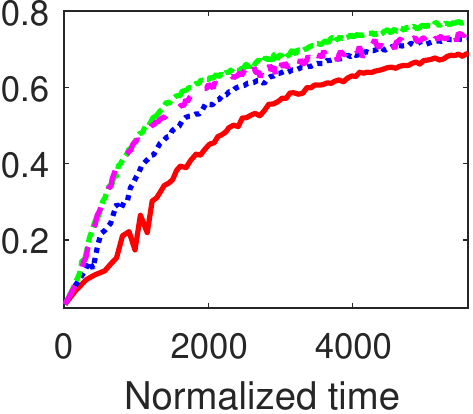}
\end{subfigure}
\begin{subfigure}[b]{0.22\columnwidth}
 \centering
 \includegraphics[width=\textwidth]{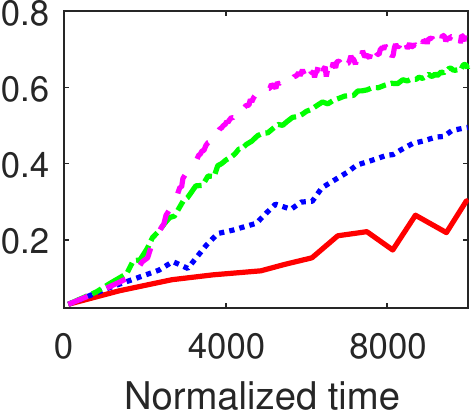}
\end{subfigure}
\caption{\label{fig:emnist_loss_accu} Performance of adaptive $k$ with proposed online learning method in Algorithm~\ref{alg:extendedOnlineLearning} (dataset: FEMNIST).}
\end{minipage}
\vspace{-0.05in}

\begin{minipage}[b]{1\columnwidth}
\centering

\begin{subfigure}[b]{0.24\columnwidth}
 \centering
 \includegraphics[width=\textwidth]{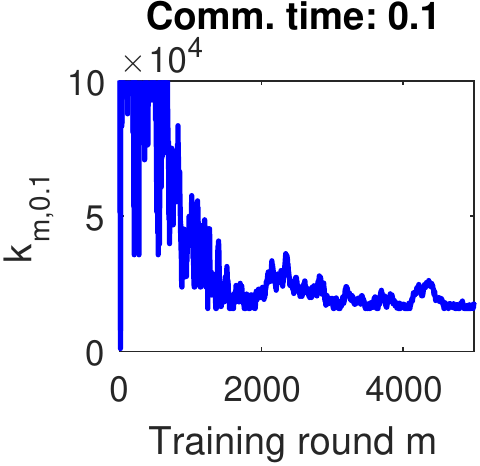}
\end{subfigure}
\begin{subfigure}[b]{0.24\columnwidth}
 \centering
 \includegraphics[width=\textwidth]{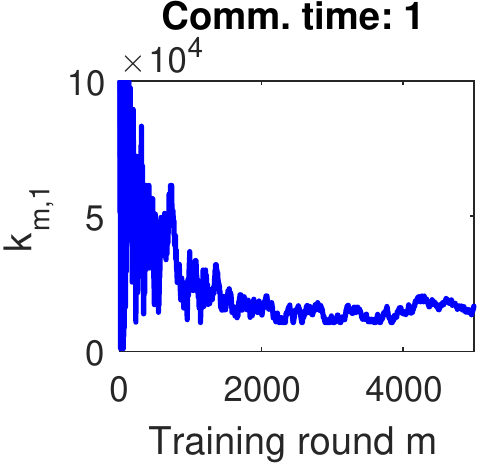}
\end{subfigure}
\begin{subfigure}[b]{0.24\columnwidth}
 \centering
 \includegraphics[width=\textwidth]{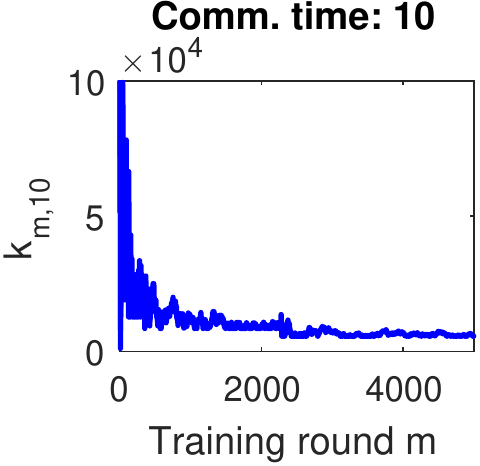}
\end{subfigure}
\begin{subfigure}[b]{0.24\columnwidth}
 \centering
 \includegraphics[width=\textwidth]{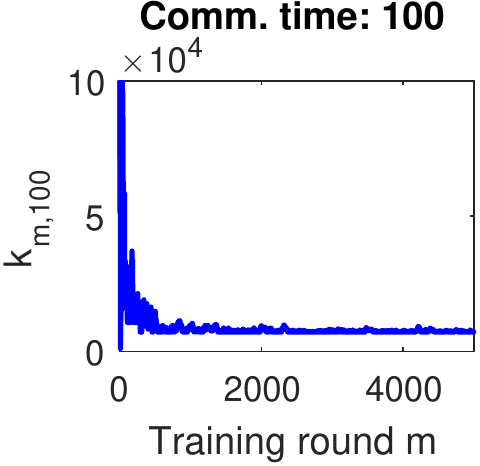}
\end{subfigure}

 \begin{subfigure}[b]{0.5\columnwidth}
 \centering
 \includegraphics[width=\textwidth]{legend_diffComm.pdf}
\end{subfigure}
~
~

\begin{subfigure}[b]{0.255\columnwidth}
 \centering
 \includegraphics[width=\textwidth]{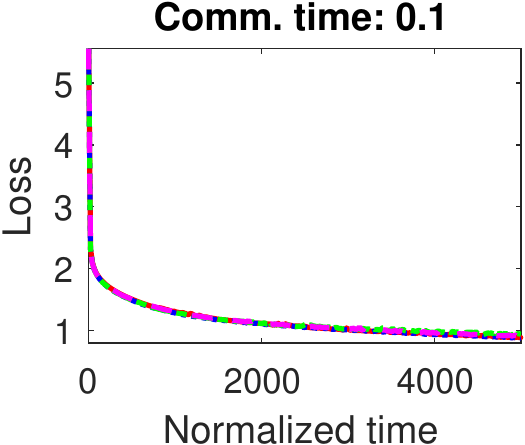}
\end{subfigure}
\begin{subfigure}[b]{0.22\columnwidth}
 \centering
 \includegraphics[width=\textwidth]{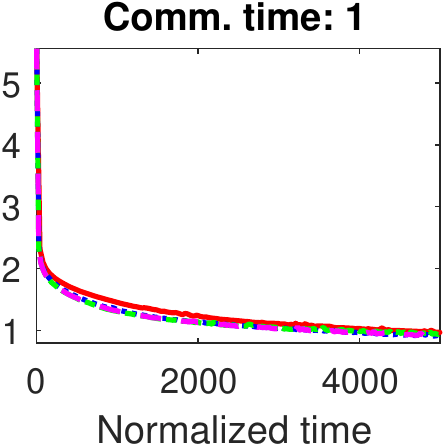}
\end{subfigure}
\begin{subfigure}[b]{0.22\columnwidth}
 \centering
 \includegraphics[width=\textwidth]{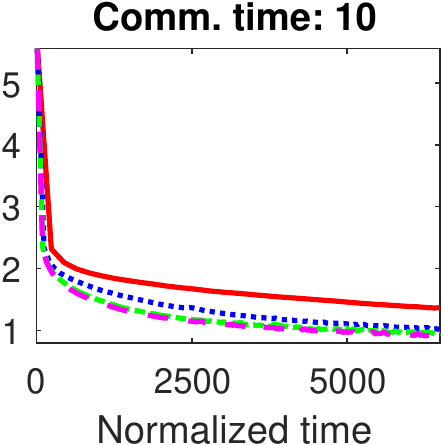}
\end{subfigure}
\begin{subfigure}[b]{0.22\columnwidth}
 \centering
 \includegraphics[width=\textwidth]{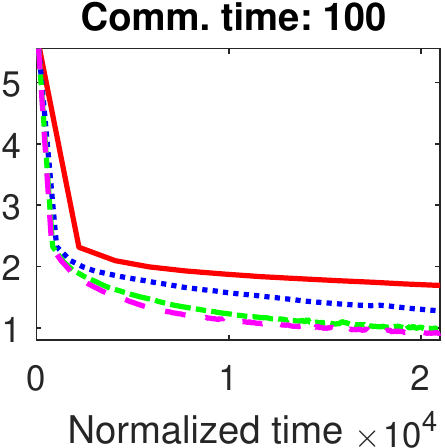}
\end{subfigure}
\vspace{-0.06in}

\begin{subfigure}[b]{0.242\columnwidth}
 \centering
 \includegraphics[width=\textwidth]{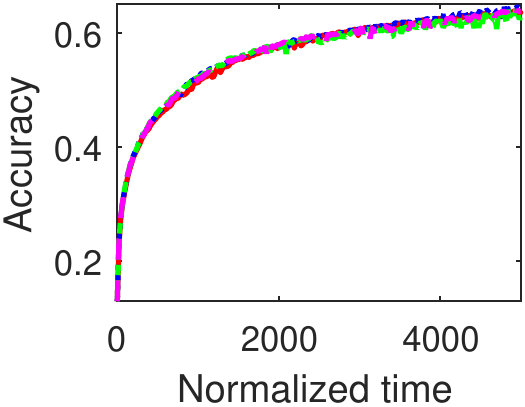}
\end{subfigure}
\begin{subfigure}[b]{0.22\columnwidth}
 \centering
 \includegraphics[width=\textwidth]{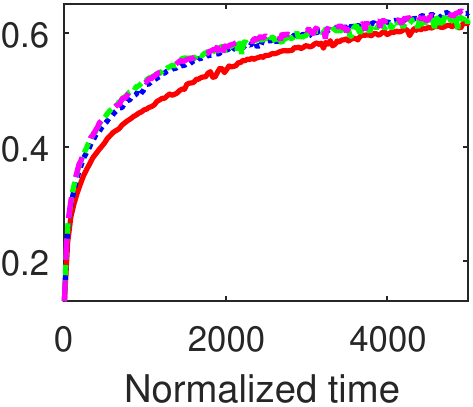}
\end{subfigure}
\begin{subfigure}[b]{0.22\columnwidth}
 \centering
 \includegraphics[width=\textwidth]{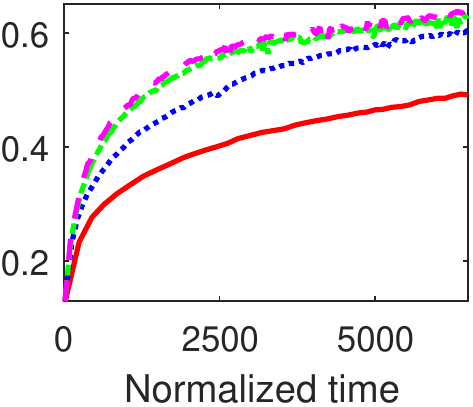}
\end{subfigure}
\begin{subfigure}[b]{0.22\columnwidth}
 \centering
 \includegraphics[width=\textwidth]{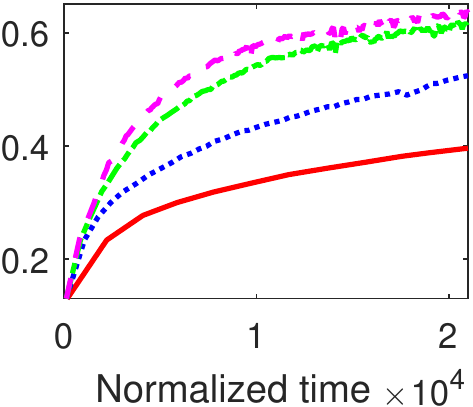}
\end{subfigure}
\caption{\label{fig:cifar_loss_accu} Performance of adaptive $k$ with proposed online learning method in Algorithm~\ref{alg:extendedOnlineLearning} (dataset: CIFAR-10).}
\end{minipage}
\end{figure}

\subsection{Performance of Online Learning for Adaptive $k$}
\label{subsec:performance-online}

We now apply the adaptive $k$ algorithm to FAB-top-$k$. We first compare our proposed approach (Algorithm~\ref{alg:extendedOnlineLearning}) with:
\begin{enumerate}
    \item \emph{Value-based gradient (derivative) descent}~\cite{hazan2016introduction}, where the derivative is estimated as in Section~\ref{subsec:overallProcedure} but without $\textrm{sign}(\cdot )$ operation and the update step size is $\delta_m$;
    \item \emph{EXP3} algorithm for MAB setting~\cite{auer2002nonstochastic}, where each integer value of $k$ is an arm in the bandit problem;
    \item \emph{Continuous bandit} setting~\cite{Online-convex-optimization-in-the-bandit}. 
\end{enumerate}
For our approach, we set $\alpha = 1.5$, $M_u = 20$, $k_\textrm{min}=0.002 \cdot D$, $k_\textrm{max}=D$. Parameters in the other approaches are set according to the same search range of $k$.
We see in Fig.~\ref{fig:diffMethods} that our proposed approach gives a better performance compared to all the other approaches and also a much more stable value of $k$ compared to EXP3 and continuous bandit.

The comparison between our proposed Algorithms~\ref{alg:onlineLearning} and \ref{alg:extendedOnlineLearning} with a large communication time of $100$ is shown in Fig.~\ref{fig:compSignOrig}, where we see that the extended approach in Algorithm~\ref{alg:extendedOnlineLearning} gives better performance and lower fluctuation in the values of $k$.

We now consider four different communication times, including $0.1$, $1$, $10$, and $100$. Let $\{k_{m,0.1}\}$, $\{k_{m,1}\}$, $\{k_{m,10}\}$, and $\{k_{m,100}\}$ denote the sequences of $k_m$ given by our proposed Algorithm~\ref{alg:extendedOnlineLearning} for each of these communication times, respectively. Figs.~\ref{fig:emnist_loss_accu} and \ref{fig:cifar_loss_accu} show the sequences of $k_m$ and the loss and accuracy values when applying different sequences of $k_m$ to each communication time, for FEMNIST and \mbox{CIFAR-10} datasets, respectively. 
In general, our algorithm uses a larger $k_m$ for a smaller communication time, as intuitively expected. For a specific communication time denoted by $\beta$, the sequence $\{k_{m,\beta}\}$ that is obtained for the same communication time $\beta$ gives the best performance\footnote{When the communication time is small with the CIFAR-10 dataset, the difference in loss and accuracy for different sequences of $k$ is small, because the way we assign samples to clients for CIFAR-10 dataset is highly non-i.i.d. and a relatively large value of $k$ is required even if for large communication time such as $100$, causing the difference between $\{k_{m,0.1}\}$, $\{k_{m,1}\}$, $\{k_{m,10}\}$, and $\{k_{m,100}\}$ to be smaller than for FEMNIST dataset.}. For example, in Fig.~\ref{fig:emnist_loss_accu}, when the communication time is $0.1$, $\{k_{m,0.1}\}$ gives a better performance than $\{k_{m,100}\}$; when the communication time is $100$, $\{k_{m,100}\}$ gives a better performance than $\{k_{m,0.1}\}$.  This shows that it is useful to adapt $k$ according to the communication/computation time and data/model characteristics; a single value (or sequence) of $k$ does not work well for all cases.

\section{Conclusion}\label{sec:conclusion}
In this paper, we have studied communication-efficient FL with adaptive GS. We have presented a FAB-top-$k$ approach which guarantees that each client provides at least $\lfloor k/N \rfloor$ gradient elements. To minimize the overall training time, we proposed a novel online learning formulation and algorithm using estimated derivative sign and adjustable search interval for determining the optimal value of $k$. Theoretical analysis of the algorithms and experimentation results using real-world datasets verify the effectiveness and benefits of our approaches over other existing techniques.

By replacing training time with another type of additive resource (e.g., energy, monetary cost), our online learning algorithm can be directly extended to the minimization of other resource consumption. Our proposed approach potentially also applies to other model compression techniques beyond GS, such as~\cite{caldas2018expanding,jiang2019model}. Future work can also consider heterogeneous client resources, where it may be beneficial to select a subset of clients in each training round and choose different $k$ for different clients, as well as the impact of GS on privacy leakage and its interplay with secure multi-party computation methods. The online learning framework proposed in this paper also sets a foundation for a broad range of optimization problems in federated and distributed learning systems.

\appendix

\subsection{Proof of Proposition~\ref{prop:tExistence}}
\label{appendix:ProofProp1}

Let $\gamma_k$ denote the time of an arbitrary training round (starting at an arbitrary loss $L'$) when using top-$k$ GS, and assume that the function $L(L')$ (for any $L'\leq L_0$) denoting the loss at the end of this training round is given for the same $k$ under consideration. Definition~\ref{def:learningTime} requires that
$
\gamma_k = \int_{L(L')}^{L'} \tilde{t}(k,l) dl = \int_a^{L'} \tilde{t}(k,l) dl - \int_a^{L(L')} \tilde{t}(k,l) dl
$,
where $a$ is an arbitrary constant. Taking the derivative w.r.t. $L'$ on both sides, we have
$ 0 = \tilde{t}(k,L') - \tilde{t}(k,L)\cdot \frac{dL}{dL'}$ 
which is equivalent to 
$\tilde{t}(k,L) = \tilde{t}(k,L') \cdot \frac{dL'}{dL}$.

When $L'=L_0$ (i.e., at model initialization), $\tilde{t}(k,l)$ for $l\in[L(L_0), L_0]$ can be constructed arbitrarily such that Definition~\ref{def:learningTime} holds. For $l < L(L_0)$, $\tilde{t}(k,l)$ can be defined recursively using $\tilde{t}(k,L) = \tilde{t}(k,L') \cdot \frac{dL'}{dL}$. 
Repeating this process for all $k$ proves the result.

\subsection{Proof of Theorem~\ref{theorem:RegretBound}}
\label{appendix:ProofTheorem1}

\begin{lemma} \label{lemma:sProductLargerThanZero}
For any $m=1,2,...,M$, we have $s_m (k_m \! -\! k^*) \!\geq\! 0$.
\end{lemma}
\begin{prf}
According to Items \emph{a)} and \emph{c)} in Assumption~\ref{assumption:convexity}, $\tau_m(k)$ is convex in $k$, and $k^*$ minimizes $\tau_m(k)$ for any $m$. Hence, we have $s_m \geq 0$ if $k_m \geq k^*$ and $s_m \leq 0$ if $k_m \leq k^*$, thus $s_m (k_m - k^*) \geq 0$ for all $m$.
\end{prf}

\begin{lemma} \label{lemma:performanceAnalysisNonRandomConvex1}
For any $m=1,2,...,M$, we have
\begin{align}
\tau_m(k_m) - \tau_m(k^*) \leq G s_m  \cdot (k_m - k^*). \label{eq:performanceAnalysisNonRandomConvex1} 
\end{align}
\end{lemma}
\begin{prf}
Due to the convexity of $t_m(\cdot)$, we have
$
\tau_m(k_m) - \tau_m(k^*) \leq \tau'_m(k_m) \cdot (k_m - k^*)  =  s_m \left| \tau'_m(k_m)  \right| \cdot (k_m - k^*)  
$
The result follows by noting that $\left| \tau'_m(k_m)  \right| \leq G$ according to~(\ref{eq:GDef}) and $s_m (k_m - k^*) \geq 0$ from Lemma~\ref{lemma:sProductLargerThanZero}.
\end{prf}

\begin{lemma} \label{lemma:performanceAnalysisNonRandom1}
For any $m=1,2,...,M$, we have
\begin{align}
s_m (k_m - k^*) & \leq \frac{(k_m - k^*)^2 - (k_{m+1} - k^*)^2}{2\delta_m} + \frac{\delta_m}{2}
\label{eq:performanceAnalysisNonRandom1}
\end{align}
where $k_{m+1} := \mathcal{P}_\mathcal{K}(k_m - \delta_m s_m)$ for all $m=1,2,...,M$. 
\end{lemma}
\begin{prf}
We note that
{\small
\begin{align*}
(k_{m+1} - k^*)^2 & = \left( \mathcal{P}_\mathcal{K}(k_m - \delta_m s_m) - k^*  \right)^2 \\
& \leq ( k_m - \delta_m s_m - k^* )^2 \tag{$k^* \in \mathcal{K}$ by definition} \\
& = (k_m - k^*)^2 + \delta_m^2 s_m^2 -2\delta_m s_m (k_m - k^*) \\
& \leq (k_m - k^*)^2 + \delta_m^2  -2\delta_m s_m (k_m - k^*) \tag{$s_m \in \{-1,0,1\}$, thus $s_m^2 \leq 1$}
\end{align*}
}%
Rearranging the inequality gives the result.
\end{prf}

Note that the definition of $k_m$ in Lemma~\ref{lemma:performanceAnalysisNonRandom1} includes $k_{M+1}$ for analysis later, although Algorithm~\ref{alg:onlineLearning} stops at $m=M$.

\begin{proof}[Proof of Theorem~\ref{theorem:RegretBound}]
Combining Lemmas~\ref{lemma:performanceAnalysisNonRandomConvex1} and \ref{lemma:performanceAnalysisNonRandom1}, we have
{\small
\begin{align}
R(M) &= \sum_{m=1}^M \left( \tau_m(k_m) - \tau_m(k^*) \right) \nonumber \\
& \leq G \sum_{m=1}^M  \frac{(k_m - k^*)^2 - (k_{m+1} - k^*)^2}{2\delta_m} + \frac{G}{2}\sum_{m=1}^M \delta_m \nonumber \\
& \leq G \sum_{m=1}^M (k_m - k^*)^2 \left( \frac{1}{2\delta_m} - \frac{1}{2\delta_{m-1}}\right)+ \frac{G}{2}\sum_{m=1}^M \delta_m \tag{$\frac{1}{\delta_0} := 0$, $-(k_{M+1} - k^*)^2 \leq 0$} \nonumber \\
& \leq  G B^2 \sum_{m=1}^M  \left( \frac{1}{2\delta_m} - \frac{1}{2\delta_{m-1}}\right)+ \frac{G}{2}\sum_{m=1}^M \delta_m \nonumber \tag{$0 \leq (k_m - k^*)^2 \leq B^2$, $\frac{1}{2\delta_m} - \frac{1}{2\delta_{m-1}} > 0$} \\
& = \frac{G B^2}{2 \delta_M} + \frac{G}{2}\sum_{m=1}^M \delta_m 
\leq G B \sqrt{2M} \nonumber
\end{align}
}%
where the last inequality is because $\delta_m := \frac{B}{\sqrt{2m}}$ and $\sum_{m=1}^M \frac{1}{\sqrt{m}} \leq 2 \sqrt{M}$.
\end{proof}

\subsection{Proof of Theorem~\ref{theorem:ExpectedRegretBound}}
\label{appendix:ProofTheorem2}

We have
{\small
\begin{align}
\tau_m(k_m) \! -\! \tau_m(k^*) 
& \leq G s_m  \cdot (k_m - k^*) \label{eq:performanceAnalysisRandomConvex} \\
& = G H_m \!\cdot\! \Expect[\hat{s}_m | k_1,...,k_m] \!\cdot\! (k_m \!-\! k^*) \label{eq:performanceAnalysisRandom2}\\
& \leq G H \cdot \Expect[\hat{s}_m | k_1,...,k_m] \cdot (k_m\! - \! k^*)  \label{eq:performanceAnalysisRandom3}\\
& = G H \cdot \Expect[\hat{s}_m  (k_m - k^*) | k_1,...,k_m]  \label{eq:performanceAnalysisRandom4}
\end{align}}%
where (\ref{eq:performanceAnalysisRandomConvex}) is from Lemma~\ref{lemma:performanceAnalysisNonRandomConvex1}; (\ref{eq:performanceAnalysisRandom2}) follows from (\ref{eq:sExpect}); (\ref{eq:performanceAnalysisRandom3}) is from $1 \leq H_m \leq H$ and $\Expect[\hat{s}_m | k_1,...,k_m] \cdot (k_m - k^*) \geq 0$, because $\Expect[\hat{s}_m | k_1,..,k_m]$ has the same sign as $s_m$, and $s_m (k_m - k^*) \geq 0$ (Lemma~\ref{lemma:sProductLargerThanZero}); (\ref{eq:performanceAnalysisRandom4}) is obtained by the property of conditional expectation that $\Expect[XY|Y] = Y\Expect[X|Y]$ for any random variables $X$ and $Y$.

Then, the expected regret is equal to
{\small
\begin{align}
&\!\!\!\!\! \Expect[R(M)]  = \Expect\left[ \sum_{m=1}^M \left( \tau_m(k_m) - \tau_m(k^*) \right) \right] \nonumber \\
&\!\!\!\!\!\! \leq \! \Expect \!\left[ \sum_{m=1}^M \! G H \!\cdot\! \Expect \! \left[\! \frac{(k_m \! -\! k^*)^2 \! -\! (k_{m+1}\! -\! k^*)^2}{2\delta_m} \! +\! \frac{\delta_m}{2} \bigg| k_1,\! ...,k_m \!\right] \!\! \right] \label{eq:performanceAnalysisRandomFinal1} \\
&\!\!\!\!\!\! \leq G H \sum_{m=1}^M \Expect\left[ \frac{(k_m - k^*)^2 - (k_{m+1} - k^*)^2}{2\delta_m} + \frac{\delta_m}{2} \right] \label{eq:performanceAnalysisRandomFinal2}  \\
&\!\!\!\!\!\! = G   H  \!\cdot\! \Expect\!\left[ \sum_{m=1}^M \!\frac{(k_m\! -\! k^*)^2 \!-\! (k_{m+1} \!-\! k^*)^2}{2\delta_m} \!\right] \!\! +\!  \frac{G H}{2} \!\!\sum_{m=1}^M \!\! \delta_m  \label{eq:performanceAnalysisRandomFinal3}  \\
&\!\!\!\!\!\! \leq  G H B \sqrt{2M}  \label{eq:performanceAnalysisRandomFinal4}  
\end{align}}%
where (\ref{eq:performanceAnalysisRandomFinal1}) is from (\ref{eq:performanceAnalysisRandom4}) and replacing $s_m$ with $\hat{s}_m$ in Lemma~\ref{lemma:performanceAnalysisNonRandom1} (it is easy that the same result of Lemma~\ref{lemma:performanceAnalysisNonRandom1} holds after this replacement); 
(\ref{eq:performanceAnalysisRandomFinal2}) is obtained by the linearity of expectation and the law of total expectation;
(\ref{eq:performanceAnalysisRandomFinal3}) is from the linearity of expectation and that $\delta_m$ is deterministic;  (\ref{eq:performanceAnalysisRandomFinal4}) is obtained by a similar procedure as in the proof of Theorem~\ref{theorem:RegretBound}.

\bibliographystyle{IEEEtran}
\bibliography{ref}

\begin{thebibliography}{10}
\providecommand{\url}[1]{#1}
\csname url@samestyle\endcsname
\providecommand{\newblock}{\relax}
\providecommand{\bibinfo}[2]{#2}
\providecommand{\BIBentrySTDinterwordspacing}{\spaceskip=0pt\relax}
\providecommand{\BIBentryALTinterwordstretchfactor}{4}
\providecommand{\BIBentryALTinterwordspacing}{\spaceskip=\fontdimen2\font plus
\BIBentryALTinterwordstretchfactor\fontdimen3\font minus
  \fontdimen4\font\relax}
\providecommand{\BIBforeignlanguage}[2]{{%
\expandafter\ifx\csname l@#1\endcsname\relax
\typeout{** WARNING: IEEEtran.bst: No hyphenation pattern has been}%
\typeout{** loaded for the language `#1'. Using the pattern for}%
\typeout{** the default language instead.}%
\else
\language=\csname l@#1\endcsname
\fi
#2}}
\providecommand{\BIBdecl}{\relax}
\BIBdecl

\bibitem{kairouz2019advances}
P.~Kairouz, H.~B. McMahan \emph{et~al.}, ``Advances and open problems in
  federated learning,'' \emph{arXiv preprint arXiv:1912.04977}, 2019.

\bibitem{mcmahan2016communication}
H.~B. McMahan, E.~Moore, D.~Ramage, S.~Hampson, and B.~A. y~Arcas,
  ``Communication-efficient learning of deep networks from decentralized
  data,'' in \emph{AISTATS}, 2017.

\bibitem{li2019federated}
T.~Li, A.~K. Sahu, A.~Talwalkar, and V.~Smith, ``Federated learning:
  Challenges, methods, and future directions,'' \emph{arXiv preprint
  arXiv:1908.07873}, 2019.

\bibitem{park2019wireless}
J.~Park, S.~Samarakoon, M.~Bennis, and M.~Debbah, ``Wireless network
  intelligence at the edge,'' \emph{Proceedings of the IEEE}, vol. 107, no.~11,
  pp. 2204--2239, 2019.

\bibitem{yang2019federated}
Q.~Yang, Y.~Liu, T.~Chen, and Y.~Tong, ``Federated machine learning: Concept
  and applications,'' \emph{ACM Transactions on Intelligent Systems and
  Technology (TIST)}, vol.~10, no.~2, p.~12, 2019.

\bibitem{GPUGrowth}
\BIBentryALTinterwordspacing
M.~McHugh, ``{GPUs} are the new star of moore's law, nvidia channel boss
  claims,'' 2018. [Online]. Available:
  \url{https://www.channelweb.co.uk/crn-uk/news/3032004/gpus-are-the-new-star-of-moores-law-nvidia-channel-boss-claims}
\BIBentrySTDinterwordspacing

\bibitem{MobileGPU}
\BIBentryALTinterwordspacing
A.~Wong, ``The mobile {GPU} comparison guide rev. 18.2,'' 2018. [Online].
  Available:
  \url{https://www.techarp.com/computer/mobile-gpu-comparison-guide/}
\BIBentrySTDinterwordspacing

\bibitem{ALinear-Speedup-Analysis-of-Distributed-Deep-Learning}
P.~Jiang and G.~Agrawal, ``A linear speedup analysis of distributed deep
  learning with sparse and quantized communication,'' in \emph{NeurIPS}, 2018.

\bibitem{ICDCS19_DIOT}
T.~D. Nguyen, S.~Marchal, M.~Miettinen \emph{et~al.}, ``Guardiot: A federated
  self-learning anomaly detection system for {IoT},'' in \emph{IEEE ICDCS},
  2019.

\bibitem{GuangxuZhu-Towards-an-Intelligent-Edge}
G.~Zhu, D.~Liu, Y.~Du \emph{et~al.}, ``Towards an intelligent edge: Wireless
  communication meets machine learning,'' \emph{arXiv preprint
  arXiv:1809.00343}, 2018.

\bibitem{INFOCOM19-Round-Robin-Synchronization}
C.~{Chen}, W.~{Wang}, and B.~{Li}, ``Round-robin synchronization: Mitigating
  communication bottlenecks in parameter servers,'' in \emph{IEEE INFOCOM},
  2019.

\bibitem{MG-WFBP}
S.~{Shi}, X.~{Chu}, and B.~{Li}, ``{MG-WFBP}: Efficient data communication for
  distributed synchronous {SGD} algorithms,'' in \emph{IEEE INFOCOM}, 2019.

\bibitem{Poseidon}
H.~Zhang, Z.~Zheng, S.~Xu \emph{et~al.}, ``Poseidon: An efficient communication
  architecture for distributed deep learning on {GPU} clusters,'' in
  \emph{USENIX ATC}, 2017.

\bibitem{ScalingDeepLearning}
Y.~You, A.~Bulu\c{c}, and J.~Demmel, ``Scaling deep learning on {GPU} and
  knights landing clusters,'' in \emph{International Conference for High
  Performance Computing, Networking, Storage and Analysis}, 2017.

\bibitem{Gaia2017}
K.~Hsieh, A.~Harlap, N.~Vijaykumar \emph{et~al.}, ``Gaia: Geo-distributed
  machine learning approaching {LAN} speeds,'' in \emph{USENIX NSDI}, 2017.

\bibitem{WangJSAC2019}
S.~{Wang}, T.~{Tuor}, T.~{Salonidis} \emph{et~al.}, ``Adaptive federated
  learning in resource constrained edge computing systems,'' \emph{IEEE Journal
  on Selected Areas in Communications}, vol.~37, no.~6, pp. 1205--1221, 2019.

\bibitem{WangSysML2019}
J.~Wang and G.~Joshi, ``Adaptive communication strategies to achieve the best
  error-runtime trade-off in local-update {SGD},'' in \emph{SysML}, 2019.

\bibitem{INFOCOM19-federated-wireless}
N.~H. {Tran}, W.~{Bao}, A.~{Zomaya}, N.~{Minh N.H.}, and C.~S. {Hong},
  ``Federated learning over wireless networks: Optimization model design and
  analysis,'' in \emph{IEEE INFOCOM}, 2019.

\bibitem{CMFL}
L.~Wang, W.~Wang, and B.~Li, ``{CMFL}: Mitigating communication overhead for
  federated learning,'' in \emph{IEEE ICDCS}, 2019.

\bibitem{Gradient-Sparsification}
J.~Wangni, J.~Wang, J.~Liu, and T.~Zhang, ``Gradient sparsification for
  communication-efficient distributed optimization,'' in \emph{NeurIPS}, 2018.

\bibitem{Qsparse-local-SGD}
D.~Basu, D.~D. abd Can~Karakus, and S.~Diggavi, ``Qsparse-local-{SGD}:
  Distributed {SGD} with quantization, sparsification, and local
  computations,'' \emph{arXiv preprint arXiv:1901.04359}, 2019.

\bibitem{Deep-Gradient-Compression}
Y.~Lin, S.~Han, H.~Mao, Y.~Wang, and W.~J. Dally, ``Deep gradient compression:
  Reducing the communication bandwidth for distributed training,'' in
  \emph{ICLR}, 2018.

\bibitem{Sparse-Communication-for-Distributed-Gradient-Descent}
A.~F. Aji and K.~Heafield, ``Sparse communication for distributed gradient
  descent,'' in \emph{Proceedings of Empirical Methods in Natural Language
  Processing}, 2017, pp. 440--445.

\bibitem{hardy2017distributed}
C.~Hardy, E.~Le~Merrer, and B.~Sericola, ``Distributed deep learning on
  edge-devices: feasibility via adaptive compression,'' in \emph{IEEE NCA},
  2017.

\bibitem{The-Convergence-of-Sparsified-Gradient-Methods}
D.~Alistarh, T.~Hoefler, M.~Johansson \emph{et~al.}, ``The convergence of
  sparsified gradient methods,'' in \emph{NeurIPS}, 2018, pp. 5977--5987.

\bibitem{chen2018adacomp}
C.-Y. Chen, J.~Choi, D.~Brand \emph{et~al.}, ``Adacomp: Adaptive residual
  gradient compression for data-parallel distributed training,'' in
  \emph{AAAI}, 2018.

\bibitem{shi2019layer}
S.~Shi, Z.~Tang, Q.~Wang, K.~Zhao, and X.~Chu, ``Layer-wise adaptive gradient
  sparsification for distributed deep learning with convergence guarantees,''
  \emph{arXiv preprint arXiv:1911.08727}, 2019.

\bibitem{gtopk-Sparsification}
S.~Shi, Q.~Wang, K.~Zhao \emph{et~al.}, ``A distributed synchronous {SGD}
  algorithm with global top-k sparsification for low bandwidth networks,'' in
  \emph{IEEE ICDCS}, 2019.

\bibitem{gtopk-Convergence}
S.~Shi, K.~Zhao, Q.~Wang, Z.~Tang, and X.~Chu, ``A convergence analysis of
  distributed {SGD} with communication-efficient gradient sparsification,'' in
  \emph{IJCAI}, 2019.

\bibitem{konevcny2016federated2}
J.~Konečný, H.~B. McMahan, F.~X. Yu \emph{et~al.}, ``Federated learning:
  Strategies for improving communication efficiency,'' in \emph{NeurIPS
  Workshop on Private Multi-Party Machine Learning}, 2016.

\bibitem{Sattler2019}
F.~{Sattler}, S.~{Wiedemann}, K.~{Müller}, and W.~{Samek}, ``Robust and
  communication-efficient federated learning from non-i.i.d. data,'' \emph{IEEE
  Transactions on Neural Networks and Learning Systems}, Nov. 2019.

\bibitem{caldas2018expanding}
S.~Caldas, J.~Kone{\v{c}}ny, H.~B. McMahan, and A.~Talwalkar, ``Expanding the
  reach of federated learning by reducing client resource requirements,''
  \emph{arXiv preprint arXiv:1812.07210}, 2018.

\bibitem{jiang2019model}
Y.~Jiang, S.~Wang, B.~J. Ko, W.-H. Lee, and L.~Tassiulas, ``Model pruning
  enables efficient federated learning on edge devices,'' \emph{arXiv preprint
  arXiv:1909.12326}, 2019.

\bibitem{xu2019elfish}
Z.~Xu, Z.~Yang, J.~Xiong, J.~Yang, and X.~Chen, ``Elfish: Resource-aware
  federated learning on heterogeneous edge devices,'' \emph{arXiv preprint
  arXiv:1912.01684}, 2019.

\bibitem{CooperativeSGD}
J.~Wang and G.~Joshi, ``Cooperative {SGD:} {A} unified framework for the design
  and analysis of communication-efficient {SGD} algorithms,'' in \emph{ICML},
  2019.

\bibitem{hazan2016introduction}
E.~Hazan \emph{et~al.}, ``Introduction to online convex optimization,''
  \emph{Foundations and Trends{\textregistered} in Optimization}, vol.~2, no.
  3-4, pp. 157--325, 2016.

\bibitem{Online-convex-optimization-in-the-bandit}
A.~D. Flaxman, A.~T. Kalai, A.~T. Kalai, and H.~B. McMahan, ``Online convex
  optimization in the bandit setting: Gradient descent without a gradient,'' in
  \emph{ACM-SIAM Symposium on Discrete Algorithms}, 2005.

\bibitem{auer2002nonstochastic}
P.~Auer, N.~Cesa-Bianchi, Y.~Freund, and R.~E. Schapire, ``The nonstochastic
  multiarmed bandit problem,'' \emph{SIAM journal on computing}, vol.~32,
  no.~1, pp. 48--77, 2002.

\bibitem{Goodfellow-et-al-2016}
I.~Goodfellow, Y.~Bengio, and A.~Courville, \emph{Deep Learning}.\hskip 1em
  plus 0.5em minus 0.4em\relax MIT Press, 2016,
  \url{http://www.deeplearningbook.org}.

\bibitem{EMNIST}
S.~Caldas, P.~Wu, T.~Li, J.~Konecn{\'{y}}, H.~B. McMahan, V.~Smith, and
  A.~Talwalkar, ``{LEAF:} {A} benchmark for federated settings,'' \emph{arXiv
  preprint arXiv:1812.01097}, 2018.

\bibitem{gupta2015deep}
S.~Gupta, A.~Agrawal, K.~Gopalakrishnan, and P.~Narayanan, ``Deep learning with
  limited numerical precision,'' in \emph{ICML}, 2015.

\bibitem{mannor2011bandits}
S.~Mannor and O.~Shamir, ``From bandits to experts: On the value of
  side-observations,'' in \emph{NeurIPS}, 2011, pp. 684--692.

\bibitem{Caron2012}
S.~Caron, B.~Kveton, M.~Lelarge, and S.~Bhagat, ``Leveraging side observations
  in stochastic bandits,'' in \emph{UAI}, 2012.

\bibitem{convex}
S.~Bubeck, ``Convex optimization: Algorithms and complexity,''
  \emph{Foundations and trends in Machine Learning}, vol.~8, no. 3-4, 2015.

\bibitem{CIFAR10}
A.~Krizhevsky and G.~Hinton, ``Learning multiple layers of features from tiny
  images,'' University of Toronto, Tech. Rep., 2009.

\end{thebibliography}

\end{document}